\documentclass[11pt]{article}
\usepackage[a4paper,margin=1in]{geometry}
\usepackage[T1]{fontenc}
\usepackage{graphicx}
\usepackage{booktabs}
\usepackage{multirow}
\usepackage{adjustbox}
\usepackage{subcaption}
\usepackage{placeins}
\usepackage{amsmath,amssymb,amsfonts,amsthm}
\usepackage{xcolor}
\usepackage[colorlinks=true,linkcolor=black,citecolor=black,urlcolor=black]{hyperref}
\definecolor{revisionbrown}{HTML}{000000}
\definecolor{editbrown}{HTML}{000000}
\definecolor{settingsbrown}{HTML}{000000}
\DeclareRobustCommand{\brownchange}[1]{#1}
\DeclareRobustCommand{\changed}[1]{#1}
\DeclareRobustCommand{\settingschange}[1]{#1}
\newcommand{\Description}[1]{}
\newcommand{\LeakyReLU}{\operatorname{LeakyReLU}}
\newtheorem{theorem}{Theorem}
\theoremstyle{definition}
\newtheorem{remark}{Remark}
\theoremstyle{plain}

\title{Gating and Learnable Temperature for Graph Attention Mechanisms}
\author{%
\begin{minipage}{\dimexpr\textwidth-2\tabcolsep\relax}
\centering\normalsize
\href{mailto:mazhongtian@mail.nwpu.edu.cn}{Zhongtian Ma}\textsuperscript{1,2}\quad
\href{mailto:wu_hao@mail.nwpu.edu.cn}{Hao Wu}\textsuperscript{1}\quad
\href{mailto:yexin.zhang@mail.nwpu.edu.cn}{Yexin Zhang}\textsuperscript{1,2}\\[0.35em]
\href{mailto:zhangqiaosheng@pjlab.org.cn}{Qiaosheng Zhang}\textsuperscript{2}\quad
\href{mailto:cgao@nwpu.edu.cn}{Chao Gao}\textsuperscript{1}\\[0.75em]
\small
\textsuperscript{1}Northwestern Polytechnical University, Xi'an 710072, China\\
\textsuperscript{2}Shanghai Artificial Intelligence Laboratory, Shanghai 200232, China
\end{minipage}}
\date{}
\hypersetup{pdftitle={Gating and Learnable Temperature for Graph Attention Mechanisms},
  pdfauthor={Zhongtian Ma, Hao Wu, Yexin Zhang, Qiaosheng Zhang, Chao Gao}}

\begin{document}
\maketitle
\begin{abstract}
Graph attention networks learn neighbor importance through data-dependent coefficients, but standard layers lack explicit control over unreliable feature dimensions and use fixed sharpness of attention coefficient distributions. This paper proposes gated graph attention and learnable temperature for common graph attention mechanisms. Gated graph attention filters feature or message responses to reduce the influence of unreliable dimensions, while learnable temperature dynamically adjusts the sharpness of the attention coefficient distribution. We conduct experiments on five standard node classification benchmarks and five heterophilic benchmarks. On most of these benchmarks, our two modifications improve upon the corresponding graph attention backbones, and combining them can yield further gains. Controlled noise experiments further verify the behavior of the proposed mechanisms under feature perturbations. Theoretical analysis explains these results by showing that gating improves robustness when only part of the feature coordinates are reliable, while temperature is beneficial when global noise weakens the discriminability of node features.
\end{abstract}

\section{Introduction}
Web graphs connect users, products, and documents through noisy and unevenly informative links. Learning from these graphs requires deciding which neighbors to aggregate and which feature responses to retain. This work studies these decisions within local graph attention, with applications to node classification on social, product, and citation networks.

\begin{figure}[t]
\centering
\includegraphics[width=0.88\textwidth]{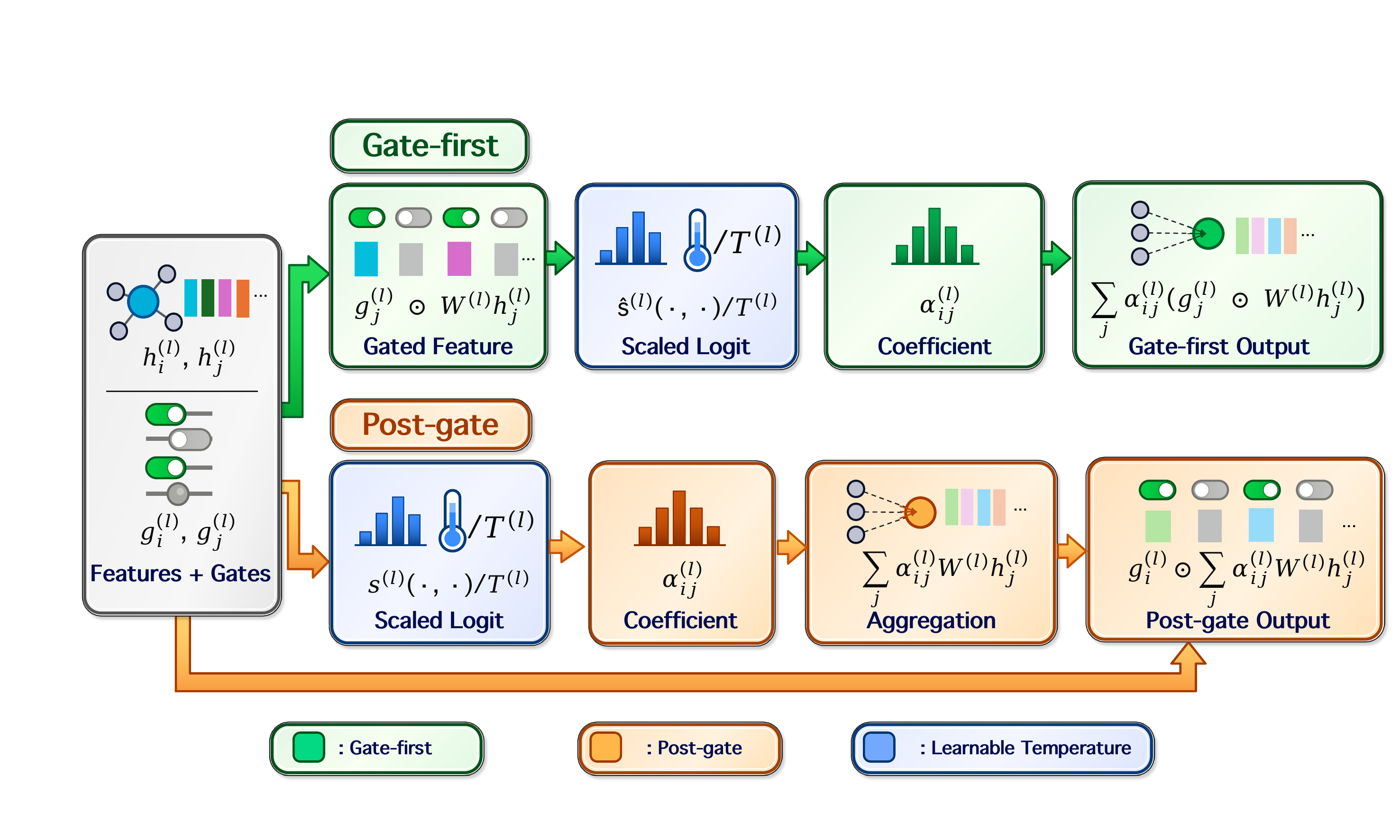}
\caption{Learnable temperature and gated graph attention. Gate-First uses gated features to compute attention coefficients and aggregate messages. Post-Gate computes coefficients from the original projected features and gates the aggregated output.}
\Description{Two processing paths compare Gate-First and Post-Gate. Both begin with features and gates. In Gate-First, feature gating precedes attention scoring and temperature scaling. In Post-Gate, attention scoring and temperature scaling precede aggregation and output gating.}
\label{fig:method-overview}
\end{figure}

Graph neural networks (GNNs) have become a widely used framework for learning representations on graph-structured data, with applications in citation networks~\cite{wu2020comprehensive}, social networks~\cite{fan2019graph}, biology~\cite{gligorijevic2021structure}, computer vision~\cite{9590574} and recommendation systems~\cite{wu2022graph}. Early graph convolutional networks (GCNs) typically aggregate neighborhood information with fixed structure-dependent weights, which makes them efficient and stable but can be limiting when neighboring nodes contribute unequally to the target node~\cite{kipf_semi-supervised_2017,hamilton_inductive_2017,rong_dropedge_2020}. Later graph attention networks (GATs) address this limitation by computing data-dependent coefficients over neighbors, so that more informative neighbors can receive larger coefficients during message passing~\cite{velickovic2018graph,wang_heterogeneous_2019}. Most of these models can be viewed under a neighborhood message-passing paradigm, where each node updates its representation by collecting and transforming information from its neighbors~\cite{feng2022powerful}.

Despite this flexibility, standard graph attention mainly adapts the relative importance of neighbors, while leaving two aspects of the attention layer less directly controlled. First, it maps unnormalized neighbor scores, or attention logits, to attention coefficients using a fixed softmax function. As a result, the sharpness of the resulting attention distribution is not explicitly adaptive. In mixed or noisy neighborhoods, this may cause the model to assign overly concentrated coefficients to unreliable neighbors, or to use an unsuitable level of smoothing across different layers and datasets~\cite{ma2025graphattention, zhu_beyond_2020}.
Second, standard graph attention offers limited control over which feature dimensions are propagated after aggregation. In high-dimensional, sparse, or partially corrupted node features, different coordinates can have very different reliability. Without an explicit feature-level filtering mechanism, noisy coordinates may still affect the attention logits or be carried into the node update~\cite{ferrini2026rethinking}.

To address these two limitations, we introduce two \changed{simple, plug-in modifications} to graph attention. The first is learnable temperature, which rescales attention logits before softmax normalization and allows each layer to adapt the sharpness of its attention distribution. The second is gated graph attention, which modulates either the aggregated node update or the projected messages used in attention aggregation, enabling the model to suppress less reliable feature responses. These two mechanisms act on different parts of the attention layer and can be applied independently or jointly to different GAT backbones. This yields a unified family of attention variants that controls not only where information is aggregated from, but also how sharply and how selectively it is propagated.

We further provide a supporting theoretical analysis to explain why these mechanisms are useful. The analysis is conducted under a two-class contextual stochastic block model (CSBM)~\cite{deshpande2018contextual}, which offers a tractable setting for studying noisy node features and attention aggregation. Specifically, we consider two common forms of feature noise, namely global perturbations that reduce the separability of node features and coordinate-wise missing or corrupted entries that leave only part of the feature vector reliable. Under global perturbations, learnable temperature can improve the node-level signal-to-noise ratio by mitigating noise-driven concentration in attention coefficients. Under coordinate-wise corruption, gated graph attention can suppress unreliable feature dimensions and improve the first-order node-level signal-to-noise ratio. These results indicate that the two mechanisms target complementary failure modes in graph attention.

Our main contributions are \changed{listed below.}
\begin{itemize}
    \item \changed{We propose learnable temperature and gated graph attention as two simple, plug-in modifications to graph attention. Learnable temperature adapts attention sharpness, while gated graph attention modulates feature-level or message-level contributions.} These mechanisms can be applied to both GAT~\cite{velickovic2018graph} and GATv2~\cite{brody2022attentive}, either independently or jointly.

    \item We provide a supporting theoretical analysis under CSBM, showing when the two mechanisms are effective. For learnable temperature, Theorem~\ref{thm:temp_gaussian} shows that it can improve attention behavior when attention coefficients are affected by \changed{uninformative} neighborhood signals. For gated graph attention, Theorem~\ref{thm:gate_missing} shows that it can be beneficial when only part of the feature dimensions is reliable or task-relevant.

    \item We evaluate the proposed variants on five standard node classification benchmarks and five heterophilic datasets. Overall, \changed{both learnable temperature and gating improve} empirical performance \changed{over} the corresponding GAT and GATv2 backbones, and the controlled noise experiments further show behavior consistent with our theoretical analysis.

\end{itemize}

\section{Related Work}
\subsection{Adaptive Attention Modulation}

Standard attention mainly controls aggregation through normalized coefficients~\cite{vaswani2017attention}, but it does not explicitly decide whether an attention output should strongly update, weakly update, or nearly skip the current representation. Gated attention addresses this limitation by adding an explicit modulation path. Bondarenko et al.~\cite{bondarenko2023quantizable} show that some Transformer heads simulate no-op or weak updates by producing extreme logits, causing activation outliers and poor quantization behavior; their gated attention applies \changed{a sigmoid gate} to each head output so that the model can directly control update magnitude. Qiu et al.~\cite{qiu2026gated} further conduct a systematic study of gated softmax attention in large language models and find that head-specific sigmoid gating after scaled dot-product attention improves performance, training stability, scaling behavior, and attention-sink mitigation. Related Transformer variants also use gates to regulate attention behavior, including content-dependent forget gates in Forgetting Transformer~\cite{lin2025forgetting}, Gated Attention Units in FLASH~\cite{hua2022transformer}, moving-average equipped gated attention in MEGA~\cite{ma2023mega}, and routing-style conditional attention in SwitchHead~\cite{csordas2024switchhead}. These works support the view that attention benefits from mechanisms that control not only where information comes from, but also how strongly an attention branch contributes.

Another line of work modulates attention by controlling the sharpness of the attention distribution through temperature or learnable scaling. In scaled dot-product attention, the factor $\sqrt{d_k}$ can be interpreted as a fixed temperature-like normalization~\cite{vaswani2017attention}, while Query-Key Normalization replaces this fixed scaling with a learnable parameter after normalizing queries and keys~\cite{henry2020query}. Adaptive sparse attention likewise investigates how learned parameters affect whether attention is more concentrated or more diffuse~\cite{correia2019adaptively}. In graph learning, Ma et al.~\cite{ma2025graphattention} show that graph attention is not always advantageous and may underperform GCN in certain regimes. This observation further motivates our use of learnable temperature to adaptively regulate attention strength in graph message passing.

\subsection{Attention-based Graph Neural Networks}

GATs extend message passing by assigning data-dependent coefficients to neighboring nodes~\cite{velickovic2018graph}. GATv2 improves the expressiveness of graph attention by replacing the static attention behavior of GAT with a more dynamic scoring mechanism~\cite{brody2022attentive}. Intranode attention has also been introduced to reweight discriminative node features within GATs~\cite{jia_look_2026}. In addition, recent work has also explored graph attention mechanisms on heterophilic and heterogeneous graphs, where attention is used to aggregate information from different types of nodes, relations, or semantic contexts~\cite{zhang_hopgat_2026, wang_heterophily_aware_2024, sun_high_frequency_2024}. Overall, this line of work uses attention as a local neighborhood aggregation rule inside GNN message passing.

Another related line is Graph Transformers, which adapt Transformer-style self-attention to graph representation learning. For example, Graphormer injects structural encodings such as centrality, spatial distance, and edge information into global self-attention~\cite{ying2021transformers}. SGFormer further simplifies this family by combining efficient global attentive propagation with a local graph propagation branch for large graphs~\cite{wu2023sgformer}. Although Graph Transformers and GAT-style models both use attention, their attention mechanisms play different roles. Graph Transformers usually perform global self-attention over graph tokens with structural biases, whereas GAT and GATv2 compute masked neighborhood attention for message passing. Our method targets the latter setting by modifying the graph attention aggregation mechanism rather than replacing the GNN with a Transformer self-attention architecture.

Gating has also been introduced into graph attention. GaAN introduces gates into multi-head graph attention by assigning a gate to each attention head, allowing the model to reweight different heads during aggregation~\cite{zhang2018gaan}. This provides an example of gated modulation in graph attention, although its gating mechanism is designed mainly for head-level combination. In contrast, the gates in our method operate on individual feature dimensions rather than attention heads, and in the Gate-First variant they are applied before the attention logits are computed, so that they also affect how the attention coefficients are formed.
SigGate-GT applies sigmoid gates to graph transformer head outputs to reduce uninformative global interactions and alleviate over-smoothing~\cite{guo2026siggate}. Our method differs from these works by jointly controlling message-passing update intensity and neighborhood attention sharpness, allowing the model to learn not only where to attend, but also how strongly and how selectively to propagate graph information.

\section{Method}\label{sec:method}

This section first reviews graph convolution and graph attention in a unified message-passing notation, and then introduces \changed{learnable temperature and gated graph attention as two modifications to graph attention.} The motivation is that standard graph attention controls the relative importance of neighbors through normalized attention coefficients, but it leaves two useful degrees of freedom less explicit. First, the sharpness of the attention distribution is fixed once the attention logits are produced, although noisy or mixed neighborhoods may require smoother or sharper aggregation. Second, the aggregated message is usually passed to the next layer without an explicit mechanism for deciding how strongly each feature dimension should contribute to the node update. We therefore use learnable temperature to adapt attention sharpness, and gated graph attention to modulate feature-level or message-level contributions. These two mechanisms act on different parts of the attention layer and can be used independently or jointly. Figure~\ref{fig:method-overview} summarizes the two modifications and their positions within the graph attention layer. \changed{Both modifications use simple operations and can be incorporated into existing graph attention layers. See Remark~\ref{remark:plug-in} for details.}

Let $G=(V,E)$ be a graph, where $V$ and $E$ denote the node and edge sets. For each node $i\in V$, let $h_i^{(l)}\in\mathbb{R}^{d_l}$ denote its representation at layer $l$, with $h_i^{(0)}=x_i$ the input feature. In graph convolution and graph attention layers, self-loops are commonly included by aggregating over $\mathcal N_i^+=\mathcal N(i)\cup\{i\}$.

\subsection{From Graph Convolution to Graph Attention}

A graph convolutional layer updates node representations by aggregating neighborhood features with fixed, structure-dependent weights\changed{ according to}
\begin{equation}
h_i^{(l+1)}
=
\phi\!\left(
\sum_{j\in\mathcal N_i^+}
\hat A_{ij}\, W^{(l)} h_j^{(l)}
\right),
\end{equation}
where $\hat A_{ij}$ is the normalized adjacency coefficient, $W^{(l)}$ is a learnable weight matrix, and $\phi(\cdot)$ is a nonlinear activation. The aggregation weights depend only on graph structure.

Graph attention replaces $\hat A_{ij}$ with data-dependent coefficients. In general form, the attention logit is
\begin{equation}
e_{ij}^{(l)} = s^{(l)}\!\left(h_i^{(l)}, h_j^{(l)}\right),
\qquad j\in\mathcal N_i^+,
\end{equation}
where $s^{(l)}(\cdot,\cdot)$ is an attention scoring function, which may include learnable projections of the endpoint features. We refer to $e_{ij}^{(l)}$ as the attention logit, i.e., the unnormalized value before softmax normalization. The normalized attention coefficient is
\begin{equation}
\alpha_{ij}^{(l)}
=
\frac{
\exp(e_{ij}^{(l)})
}{
\sum_{r\in\mathcal N_i^+} \exp(e_{ir}^{(l)})
},
\end{equation}
and the updated representation is
\begin{equation}
h_i^{(l+1)}
=
\sum_{j\in\mathcal N_i^+}
\alpha_{ij}^{(l)} W^{(l)} h_j^{(l)}.
\end{equation}

Two standard instances are GAT and GATv2.

\paragraph{GAT}
Following the original formulation, GAT computes
\begin{equation}
e_{ij}^{(l)}
=
\LeakyReLU\!\left(
(a^{(l)})^\top
\big[
W^{(l)} h_i^{(l)} \,\Vert\, W^{(l)} h_j^{(l)}
\big]
\right),
\end{equation}
where $W^{(l)}$ is a learnable feature projection matrix and $a^{(l)}$ is a learnable attention vector.

\paragraph{GATv2.}
Following GATv2, the attention logit is computed as
\begin{equation}
e_{ij}^{(l)}
=
(a^{(l)})^\top
\LeakyReLU\!\left(
W_{\mathrm{pair}}^{(l)}
\big[
h_i^{(l)} \,\Vert\, h_j^{(l)}
\big]
\right).
\end{equation} 
Unlike in GAT, here $W_{\mathrm{pair}}^{(l)}$ acts on the concatenated node pair $[h_i^{(l)} \Vert h_j^{(l)}]$\changed{, and $a^{(l)}$ is applied after the LeakyReLU nonlinearity rather than before it}. \changed{As a result}, GATv2 yields a more flexible query-dependent ranking of neighbors.

\subsection{Learnable Temperature for Graph Attention}

We introduce a learnable temperature into graph attention. For layer $l$, let $T^{(l)}>0$ be a trainable scalar. The attention coefficient becomes
\begin{equation}\label{eq:temperature}
\alpha_{ij}^{(l)}
=
\frac{
\exp\!\left(e_{ij}^{(l)}/T^{(l)}\right)
}{
\sum_{r\in\mathcal N_i^+}
\exp\!\left(e_{ir}^{(l)}/T^{(l)}\right)
}.
\end{equation}
The corresponding update is
\begin{equation}
h_i^{(l+1)}
=
\sum_{j\in\mathcal N_i^+}
\alpha_{ij}^{(l)} W^{(l)} h_j^{(l)}.
\end{equation}

The temperature rescales attention logits before normalization. Smaller $T^{(l)}$ yields sharper attention, while larger $T^{(l)}$ produces smoother aggregation.

In implementation, positivity of $T^{(l)}$ is enforced by parameterizing it as $T^{(l)}=\operatorname{softplus}(\theta^{(l)})+\epsilon$, where $\theta^{(l)}$ is unconstrained and $\epsilon>0$ is a small constant.

\subsection{Gated Graph Attention}

We further introduce a node-dependent gate. For node $i$ at layer $l$, the gate is computed from the layer input $h_i^{(l)}$\changed{ as}
\begin{equation}\label{eq:gate}
g_i^{(l)}
=
\operatorname{sigmoid}\!\left(
W_g^{(l)} h_i^{(l)} + b_g^{(l)}
\right),
\end{equation}
where $W_g^{(l)}$ and $b_g^{(l)}$ are learnable gate parameters, and $g_i^{(l)}$ has the same dimension as the projected feature $W^{(l)}h_i^{(l)}$. We consider two variants.

\paragraph{Post-Gate Attention.}
\changed{We first aggregate the messages using attention and then apply the gate to obtain}
\begin{equation}\label{eq:postgate}
h_i^{(l+1)}
=
\left(
\sum_{j\in\mathcal N_i^+}
\alpha_{ij}^{(l)} W^{(l)} h_j^{(l)}
\right)
\odot g_i^{(l)}.
\end{equation}

\paragraph{Gate-First Attention.}
The gate is first applied to the projected feature. The resulting gated feature is then used in both attention logit computation and message aggregation. Under the general attention form,
\begin{equation}\label{eq:gatefirst}
e_{ij}^{(l)}
=
\hat s^{(l)}\!\left(
W^{(l)} h_i^{(l)} \odot g_i^{(l)},
W^{(l)} h_j^{(l)} \odot g_j^{(l)}
\right),
\end{equation}
\begin{equation}
\alpha_{ij}^{(l)}
=
\frac{
\exp(e_{ij}^{(l)})
}{
\sum_{r\in\mathcal N_i^+} \exp(e_{ir}^{(l)})
},
\end{equation}
and
\begin{equation}
h_i^{(l+1)}
=
\sum_{j\in\mathcal N_i^+}
\alpha_{ij}^{(l)}
\left(
W^{(l)} h_j^{(l)} \odot g_j^{(l)}
\right).
\end{equation}
For specific operators such as GAT and GATv2, this amounts to replacing the projected feature in the original attention logit computation with its gated counterpart.

\subsection{Combining Temperature and Gating}

Temperature can be combined with both gating variants.

\paragraph{Temperature with Post-Gate Attention.}
We first scale the attention logits by temperature and normalize them,
\begin{equation}
\alpha_{ij}^{(l)}
=
\frac{
\exp\!\left(
s^{(l)}(h_i^{(l)},h_j^{(l)}) / T^{(l)}
\right)
}{
\sum_{r\in\mathcal N_i^+}
\exp\!\left(
s^{(l)}(h_i^{(l)},h_r^{(l)}) / T^{(l)}
\right)
},
\end{equation}
\changed{and then apply the gate to obtain}
\begin{equation}
h_i^{(l+1)}
=
\left(
\sum_{j\in\mathcal N_i^+}
\alpha_{ij}^{(l)} W^{(l)} h_j^{(l)}
\right)
\odot g_i^{(l)}.
\end{equation}

\paragraph{Temperature with Gate-First Attention.}
\changed{We first gate the projected features and then compute the temperature-scaled attention coefficients as}
\begin{equation}
\alpha_{ij}^{(l)}
=
\frac{
\exp\!\left(
\hat s^{(l)}(
W^{(l)} h_i^{(l)} \odot g_i^{(l)},
W^{(l)} h_j^{(l)} \odot g_j^{(l)}
)/T^{(l)}
\right)
}{
\sum_{r\in\mathcal N_i^+}
\exp\!\left(
\hat s^{(l)}(
W^{(l)} h_i^{(l)} \odot g_i^{(l)},
W^{(l)} h_r^{(l)} \odot g_r^{(l)}
)/T^{(l)}
\right)
},
\end{equation}
\changed{and use these coefficients to aggregate the gated messages as}
\begin{equation}
h_i^{(l+1)}
=
\sum_{j\in\mathcal N_i^+}
\alpha_{ij}^{(l)}
\left(
W^{(l)} h_j^{(l)} \odot g_j^{(l)}
\right).
\end{equation}

We retain the following names throughout the experiments. \texttt{Gated} uses Post-Gate with $T^{(l)}=1$, \texttt{Temp\_only} learns $T^{(l)}$ without gating, \texttt{Temp\_gated} combines temperature with Post-Gate, and \texttt{Gated\_temp} combines Gate-First with temperature. The suffix \texttt{\_v2} denotes the GATv2 backbone; names without this suffix use GAT.

\subsection{Gate Regularization}

For gated variants, we optionally impose a sparsity regularizer on gate activations\changed{, giving the training objective}
\begin{equation}
\mathcal L
=
\mathcal L_{\mathrm{task}}
+
\lambda_{\mathrm{gate}}
\sum_l \mathrm{mean}\!\left(g^{(l)}\right),
\end{equation}
where $\mathcal L_{\mathrm{task}}$ is the task loss, $\lambda_{\mathrm{gate}}$ is a regularization coefficient, and $\mathrm{mean}(g^{(l)})$ denotes the average gate value at layer $l$.

In experiments, we use the standard multi-head extension of graph attention. Unless otherwise stated, the learnable temperature is implemented as a layer-wise scalar shared across attention heads, and the gate is applied to the projected feature dimensions of each layer after head concatenation or to the corresponding projected message representation. Since temperature and gating do not change the neighborhood set or the attention normalization domain, we present the method above in its single-head form for clarity.

\begin{remark}[\changed{Simple and plug-in design}]\label{remark:plug-in}
\changed{Both mechanisms can be incorporated into existing GAT and GATv2 layers through simple operations. Learnable temperature adds one scalar parameter per layer, while feature-wise gating uses a sigmoid projection with $(d_{\mathrm{in}}+1)d_g$ parameters, where $d_{\mathrm{in}}$ is the layer input dimension and $d_g$ is the gate dimension across attention heads. Neither mechanism changes the neighborhood set, sampling procedure, or message-passing interface.}
\end{remark}

\section{Experiments}\label{sec:experiments}
We compare GCN~\cite{kipf_semi-supervised_2017}, GAT~\cite{velickovic2018graph}, GATv2~\cite{brody2022attentive}, GaAN~\cite{zhang2018gaan}, SGFormer~\cite{wu2023sgformer}, and the eight variants defined in Section~\ref{sec:method}. \brownchange{Table~\ref{tab:methods} lists their backbones, use of learnable temperature and gating, and gating positions.} GaAN provides a head-level gating comparison, while SGFormer combines global attentive propagation with a local graph branch. All methods use the same dataset splits, metrics, and validation criterion for checkpoint selection. In mini-batch experiments, SGFormer's global attention operates within each sampled subgraph. Temperatures are initialized to 1.0 and gate biases to 0.0. The gate regularization coefficient is $10^{-5}$ on Cora and Citeseer and $10^{-6}$ on the remaining benchmarks. Appendix~\ref{app:settings} summarizes the remaining implementation settings.

\noindent\textbf{AI Assistance.} We used generative AI tools to assist with code debugging, running baseline experiments, and language polishing. The authors take full responsibility for the content of this paper.

\begin{table}[t]
\centering
\caption{\brownchange{Compared methods. Temp. and Gate indicate the proposed temperature and gating mechanisms; Placement identifies the gating position.}}
\label{tab:methods}
\small
\setlength{\tabcolsep}{3pt}
\color{revisionbrown}
\begin{adjustbox}{max width=\linewidth}
\begin{tabular}{lcccl}
\toprule
Method & Backbone & Temp. & Gate & Placement \\
\midrule
GCN & GCN & No & No & -- \\
GAT & GAT & No & No & -- \\
GATv2 & GATv2 & No & No & -- \\
GaAN & GaAN & -- & -- & -- \\
SGFormer & SGFormer & -- & -- & -- \\
\midrule
Gated & GAT & No & Yes & Post-Gate \\
Temp\_only & GAT & Yes & No & -- \\
Temp\_gated & GAT & Yes & Yes & Post-Gate \\
Gated\_temp & GAT & Yes & Yes & Gate-First \\
\midrule
Gated\_v2 & GATv2 & No & Yes & Post-Gate \\
Temp\_only\_v2 & GATv2 & Yes & No & -- \\
Temp\_gated\_v2 & GATv2 & Yes & Yes & Post-Gate \\
Gated\_temp\_v2 & GATv2 & Yes & Yes & Gate-First \\
\bottomrule
\end{tabular}
\end{adjustbox}
\end{table}

\settingschange{Table~\ref{tab:dataset-statistics} summarizes the datasets used in our experiments.}
\begin{table}[!t]
\centering
\caption{\settingschange{Dataset Statistics. H2Index is reported for the five H2GB benchmarks.}}
\label{tab:dataset-statistics}
\small
\setlength{\tabcolsep}{2pt}
\color{settingsbrown}
\begin{adjustbox}{max width=\linewidth}
\begin{tabular}{lrrrrr}
\toprule
Dataset & Nodes & Edges & Features & Classes & H2Index\\
\midrule
Cora & 2,708 & 10,556 & 1,433 & 7 & ---\\
Citeseer & 3,327 & 9,104 & 3,703 & 6 & ---\\
OGBN-Arxiv & 169,343 & 1,166,243 & 128 & 40 & ---\\
OGBN-Products & 2,449,029 & 61,859,140 & 100 & 47 & ---\\
Flickr & 89,250 & 899,756 & 500 & 7 & ---\\
\midrule
Mag-year & 1,939,743 & 42,182,144 & 128 & 5 & 0.9654\\
Pokec & 1,731,977 & 51,774,836 & 66 & 2 & 0.9488\\
OAG-Chem & 1,918,881 & 38,098,014 & 768 & 2,985 & 0.8858\\
OGBN-MAG & 1,939,743 & 42,182,144 & 128 & 349 & 0.8773\\
OAG-CS & 1,112,691 & 27,537,448 & 768 & 3,514 & 0.9652\\
\bottomrule
\end{tabular}
\end{adjustbox}
\end{table}

\subsection{Experiments on Standard Node Classification Benchmarks}
We evaluate on Cora, Citeseer, OGBN-Arxiv, OGBN-Products, and Flickr. We use the official data splits and report test performance at the best validation checkpoint. Flickr uses Micro-F1, and the other benchmarks use accuracy, with the official OGB evaluator where applicable~\cite{hu2020open}. Results are means and standard deviations over 50 runs for Cora and Citeseer, and 20 runs for OGBN-Arxiv, OGBN-Products, and Flickr. Training uses Adam and cross entropy, with full-batch updates on Cora, Citeseer, and OGBN-Arxiv and neighbor sampling on OGBN-Products and Flickr. OGBN-Products uses training/evaluation batch sizes 2048/4096 and fanout $[15,10,5]$; Flickr uses 4096/8192 and $[25,10]$.

\settingschange{For GAT, GATv2, and their variants, Cora and Citeseer use two layers with eight heads and eight hidden features per head, learning rate 0.005, and 400 and 500 epochs, respectively. OGBN-Arxiv uses three layers, four heads, 64 hidden features per head, learning rate 0.005, and 1,000 epochs. OGBN-Products uses the same layer and head configuration with learning rate 0.003 and 50 epochs. Flickr uses two layers, four heads, 32 hidden features per head, learning rate 0.003, and 500 epochs. Dropout is 0, 0.5, 0.3, 0.5, and 0.5 in this order.}

\begin{table}[t]
\centering
\caption{Performance comparison on five standard node classification benchmarks. Values are test mean $\pm$ std (\%). Best and second-best results for each dataset are highlighted in bold and underlined, respectively. \texttt{Flickr} is evaluated by Micro-F1; other datasets use accuracy.}
\small
\renewcommand{\arraystretch}{1.05}
\setlength{\tabcolsep}{4pt}
\begin{adjustbox}{max width=\linewidth}

\begin{tabular}{lccccc}
\toprule
\multirow{2}{*}{Method} & \texttt{Cora} & \texttt{Citeseer} & \texttt{OGBN-Arxiv} & \texttt{OGBN-Products} & \texttt{Flickr} \\
 & Acc. (\%) & Acc. (\%) & Acc. (\%) & Acc. (\%) & Micro-F1 (\%) \\
\midrule
GCN & 80.81 $\pm$ 0.76 & 69.24 $\pm$ 0.75 & 71.06 $\pm$ 0.33 & 77.49 $\pm$ 0.17 & 51.37 $\pm$ 0.13 \\
GAT & 80.81 $\pm$ 1.07 & 68.61 $\pm$ 0.97 & 71.39 $\pm$ 0.35 & 78.87 $\pm$ 0.28 & 51.80 $\pm$ 0.16 \\
GATv2 & 80.53 $\pm$ 1.00 & 68.35 $\pm$ 0.91 & 71.59 $\pm$ 0.30 & 79.76 $\pm$ 0.33 & 52.15 $\pm$ 0.22 \\
GaAN & 78.38 $\pm$ 0.62 & 68.38 $\pm$ 1.32 & 69.77 $\pm$ 0.25 & 75.37 $\pm$ 0.74 & 51.72 $\pm$ 0.19 \\
SGFormer & 76.69 $\pm$ 1.58 & 56.69 $\pm$ 2.28 & \textbf{72.54 $\pm$ 0.11} & 74.12 $\pm$ 0.48 & \textbf{52.89 $\pm$ 0.20} \\
Gated & 80.89 $\pm$ 1.14 & 69.21 $\pm$ 1.26 & 71.89 $\pm$ 0.34 & \textbf{80.62 $\pm$ 0.19} & 52.41 $\pm$ 0.20 \\
Temp\_only & \underline{81.35 $\pm$ 0.95} & 69.20 $\pm$ 0.91 & 71.51 $\pm$ 0.33 & 77.90 $\pm$ 0.28 & 51.70 $\pm$ 0.17 \\
Temp\_gated & \textbf{81.45 $\pm$ 1.05} & \underline{69.34 $\pm$ 1.13} & 71.57 $\pm$ 0.41 & 79.08 $\pm$ 0.19 & 52.48 $\pm$ 0.34 \\
Gated\_temp & 81.34 $\pm$ 1.08 & \textbf{69.43 $\pm$ 1.00} & 71.67 $\pm$ 0.22 & 78.95 $\pm$ 0.44 & \underline{52.61 $\pm$ 0.23} \\
Gated\_v2 & 81.13 $\pm$ 1.12 & 69.25 $\pm$ 1.18 & \underline{72.04 $\pm$ 0.24} & \underline{80.09 $\pm$ 0.44} & 52.34 $\pm$ 0.24 \\
Temp\_only\_v2 & 80.91 $\pm$ 0.96 & 69.28 $\pm$ 1.06 & 71.29 $\pm$ 0.25 & 78.87 $\pm$ 0.37 & 51.86 $\pm$ 0.29 \\
Temp\_gated\_v2 & 81.21 $\pm$ 1.10 & 69.27 $\pm$ 1.22 & 71.74 $\pm$ 0.33 & 79.90 $\pm$ 0.22 & 52.30 $\pm$ 0.18 \\
Gated\_temp\_v2 & 81.21 $\pm$ 1.02 & 69.20 $\pm$ 1.14 & 71.82 $\pm$ 0.19 & 79.71 $\pm$ 0.60 & 52.58 $\pm$ 0.26 \\
\bottomrule
\end{tabular}
\end{adjustbox}

\label{tab:homo_selected_with_ogbn_products}
\end{table}

\paragraph{Individual mechanisms.}
Table~\ref{tab:homo_selected_with_ogbn_products} shows that both mechanisms provide useful improvements, but their effects are not identical. Relative to GAT, Gated improves the results on \changed{\texttt{Cora},} \texttt{Citeseer}, \texttt{OGBN-Arxiv}, \texttt{OGBN-Products}, and \texttt{Flickr}. A similar trend can be observed for the GATv2 backbone, where Gated\_v2 improves over GATv2 on \texttt{Cora}, \texttt{Citeseer}, \texttt{OGBN-Arxiv}, \texttt{OGBN-Products}, and \texttt{Flickr}. \changed{Temp\_only improves} over GAT on \texttt{Cora}, \texttt{Citeseer}, and \texttt{OGBN-Arxiv}\changed{, and Temp\_only\_v2 improves over GATv2 on \texttt{Cora} and \texttt{Citeseer}}. These results suggest that \changed{both mechanisms are beneficial overall, with gating providing more consistent gains and learnable temperature providing more} dataset-dependent benefits.

\paragraph{Combined mechanisms.}
\begingroup\emergencystretch=2em
\changed{Combining the two mechanisms gives better and more consistent overall performance. On both the standard and H2GB benchmarks, the best average rank is attained by a combined variant (Figures~\ref{fig:average-rank-standard} and~\ref{fig:average-rank-h2gb}).} Temp\_gated achieves the best performance on \texttt{Cora}, while Gated\_temp obtains the best results among our variants on \texttt{Citeseer} and \texttt{Flickr}. Thus, combined methods achieve the top score among our variants on three of the five standard node classification benchmarks. On the remaining datasets, gate-only variants can be strongest, but the combined variants remain competitive. Compared with the best proposed variant, GaAN is lower on all five datasets, whereas SGFormer is slightly higher on \texttt{OGBN-Arxiv} (72.54\% versus 72.04\%) and \texttt{Flickr} (52.89\% versus 52.61\%) and lower on the other three datasets. These comparisons are based on mean performance.\par\endgroup

\brownchange{On OGBN-Products, the best proposed variant exceeds GaAN and SGFormer by 5.25 and 6.50 percentage points, respectively.}

\begin{figure}[!htbp]
\centering
\includegraphics[width=0.78\linewidth]{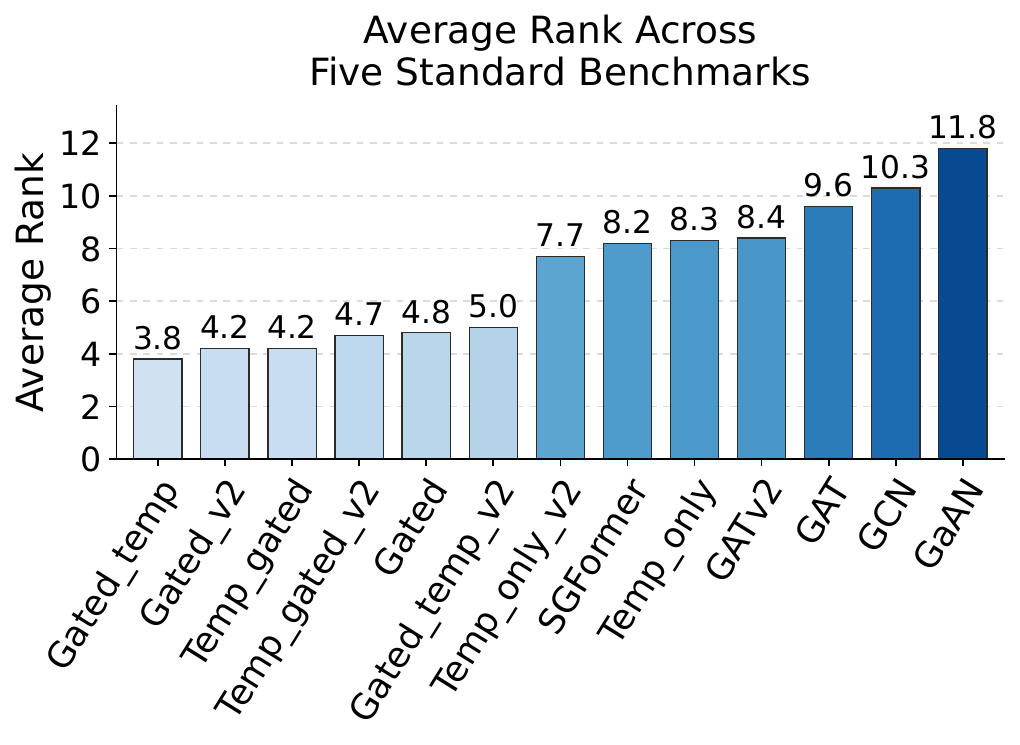}
\caption{Average ranks across the five standard benchmarks in Table~\ref{tab:homo_selected_with_ogbn_products}. Lower is better; ties receive average ranks.}
\Description{A bar chart summarizes the ranks of all thirteen methods across the five standard node classification benchmarks.}
\label{fig:average-rank-standard}
\end{figure}

\brownchange{Figure~\ref{fig:average-rank-standard} summarizes the overall comparison in Table~\ref{tab:homo_selected_with_ogbn_products}. Both gating-only variants and all four combined variants rank ahead of their corresponding backbones. Gated\_temp has the best average rank of 3.80, followed by Temp\_gated and Gated\_v2 at 4.20.}

\subsection{Experiments on Heterophilic Benchmarks}\label{sec:h2gb}
We evaluate on five H2GB benchmarks~\cite{lin_when_2025}, Mag-year, Pokec, OAG-Chem, OGBN-MAG, and OAG-CS. H2Index measures graph heterophily, with higher values indicating stronger heterophily. The values for these five datasets range from 0.8773 to 0.9654 and are taken from H2GB~\cite{lin_when_2025}, which provides the calculation formula. \changed{We convert the graphs to a homogeneous representation for message passing and focus on heterophily.} Input features are projected into a shared space, nodes without raw features receive learnable embeddings, and classification is performed on the target nodes. All methods use the same preprocessing, splits, and checkpoint rule. We train for 50 epochs with neighbor sampling, batch size 2048, and fanout $[20,10]$, and report test accuracy over ten runs.

\settingschange{For GAT, GATv2, and their variants, Mag-year and OGBN-MAG use four layers, eight heads, and hidden width 512. Pokec and OAG-Chem use three layers, eight heads, and width 256; OAG-CS uses three layers, four heads, and width 256. Dropout is 0.5 on Pokec and 0.2 elsewhere. The learning rate is 0.005 on OAG-Chem and 0.001 on the other four benchmarks.}

\begin{table}[t]
\centering
\small
\caption{Performance comparison on H2GB datasets over 10 runs. Values are test accuracy mean $\pm$ standard deviation (\%). Best and second-best means for each dataset are highlighted in bold and underlined, respectively.}
\setlength{\tabcolsep}{4pt}
\begin{adjustbox}{max width=\linewidth}
\begin{tabular}{lccccc}
\toprule
Method & \texttt{Mag-year} (\%) & \texttt{Pokec} (\%) & \texttt{OAG-Chem} (\%) & \texttt{OGBN-MAG} (\%) & \texttt{OAG-CS} (\%) \\
\midrule
GCN & 36.32 $\pm$ 0.15 & 67.12 $\pm$ 0.09 & 15.83 $\pm$ 0.35 & 42.40 $\pm$ 0.21 & 12.17 $\pm$ 0.44 \\
GAT & 37.08 $\pm$ 0.10 & 71.13 $\pm$ 0.09 & 17.34 $\pm$ 0.18 & 50.12 $\pm$ 0.99 & 15.17 $\pm$ 0.34 \\
GATv2 & 37.96 $\pm$ 0.34 & 75.22 $\pm$ 0.08 & 17.42 $\pm$ 0.33 & \underline{51.82 $\pm$ 0.26} & 15.36 $\pm$ 0.38 \\
GaAN & 38.49 $\pm$ 0.13 & 75.52 $\pm$ 0.09 & 7.63 $\pm$ 2.46 & 51.76 $\pm$ 0.14 & 11.97 $\pm$ 0.38 \\
SGFormer & 36.88 $\pm$ 0.11 & 72.86 $\pm$ 0.09 & 15.10 $\pm$ 0.37 & 45.50 $\pm$ 0.30 & 12.56 $\pm$ 0.44 \\
Gated & 37.46 $\pm$ 0.42 & 76.19 $\pm$ 0.06 & 14.38 $\pm$ 0.45 & 51.42 $\pm$ 0.16 & 15.58 $\pm$ 0.48 \\
Temp\_only & 36.43 $\pm$ 0.22 & 71.22 $\pm$ 0.13 & 17.45 $\pm$ 0.16 & \textbf{51.99 $\pm$ 0.29} & 14.88 $\pm$ 0.28 \\
Temp\_gated & 38.51 $\pm$ 0.10 & 76.18 $\pm$ 0.05 & \underline{17.95 $\pm$ 0.35} & 51.66 $\pm$ 0.15 & 15.58 $\pm$ 0.37 \\
Gated\_temp & \underline{38.53 $\pm$ 0.17} & 74.84 $\pm$ 0.12 & 17.12 $\pm$ 0.19 & 50.27 $\pm$ 0.45 & 14.30 $\pm$ 0.58 \\
Gated\_v2 & 38.38 $\pm$ 0.08 & \underline{76.21 $\pm$ 0.05} & 17.82 $\pm$ 0.35 & 51.59 $\pm$ 0.06 & \underline{15.59 $\pm$ 0.44} \\
Temp\_only\_v2 & 36.89 $\pm$ 0.34 & 73.02 $\pm$ 0.10 & 17.68 $\pm$ 0.21 & 50.97 $\pm$ 0.29 & 15.24 $\pm$ 0.37 \\
Temp\_gated\_v2 & 38.42 $\pm$ 0.06 & \textbf{76.23 $\pm$ 0.02} & \textbf{17.98 $\pm$ 0.22} & 51.61 $\pm$ 0.30 & \textbf{15.71 $\pm$ 0.41} \\
Gated\_temp\_v2 & \textbf{38.66 $\pm$ 0.28} & 75.03 $\pm$ 0.04 & 17.43 $\pm$ 0.23 & 50.96 $\pm$ 0.34 & 14.79 $\pm$ 0.51 \\
\bottomrule
\end{tabular}
\end{adjustbox}
\label{tab:h2gb_selected_combined_testacc}
\end{table}

\paragraph{Performance on Heterophilic Graphs.}
The H2Index values reported in \settingschange{Table~\ref{tab:dataset-statistics}} measure the heterophily of the five H2GB datasets. \brownchange{Table~\ref{tab:h2gb_selected_combined_testacc} shows that our best-performing variants achieve higher accuracy than GAT, GATv2, GaAN, and SGFormer on all five datasets. Figure~\ref{fig:average-rank-h2gb} summarizes these results through average ranks, with Temp\_gated\_v2 and Temp\_gated ranking first and second, respectively.} These results show that the proposed modifications to local graph attention remain competitive with GaAN and SGFormer under the same H2GB evaluation protocol.

\brownchange{The individual variants show different effects across datasets. On Pokec, Gated raises the GAT score from 71.13\% to 76.19\%, while Temp\_gated reaches 76.18\%. On OAG-Chem, Gated scores 14.38\%, below GAT at 17.34\%, whereas Temp\_gated reaches 17.95\%. On OGBN-MAG, Temp\_only attains the best reported mean of 51.99\%. These comparisons show that the benefit of combining the mechanisms depends on the dataset.}

\paragraph{\brownchange{Comparison of Combination Strategies.}}
\brownchange{Figure~\ref{fig:combined-gain-heatmap} summarizes the performance of the four combined variants relative to their corresponding GAT or GATv2 backbone across the standard and heterophilic benchmarks. Combining learnable temperature and gating improves performance on most datasets. \changed{On the heterophilic benchmarks, temperature with Post-Gate performs better overall than Gate-First with temperature. Temp\_gated and Temp\_gated\_v2 outperform Gated\_temp and Gated\_temp\_v2, respectively, on four of the five datasets.}}

\begin{figure}[!htbp]
\centering
\includegraphics[width=0.78\linewidth]{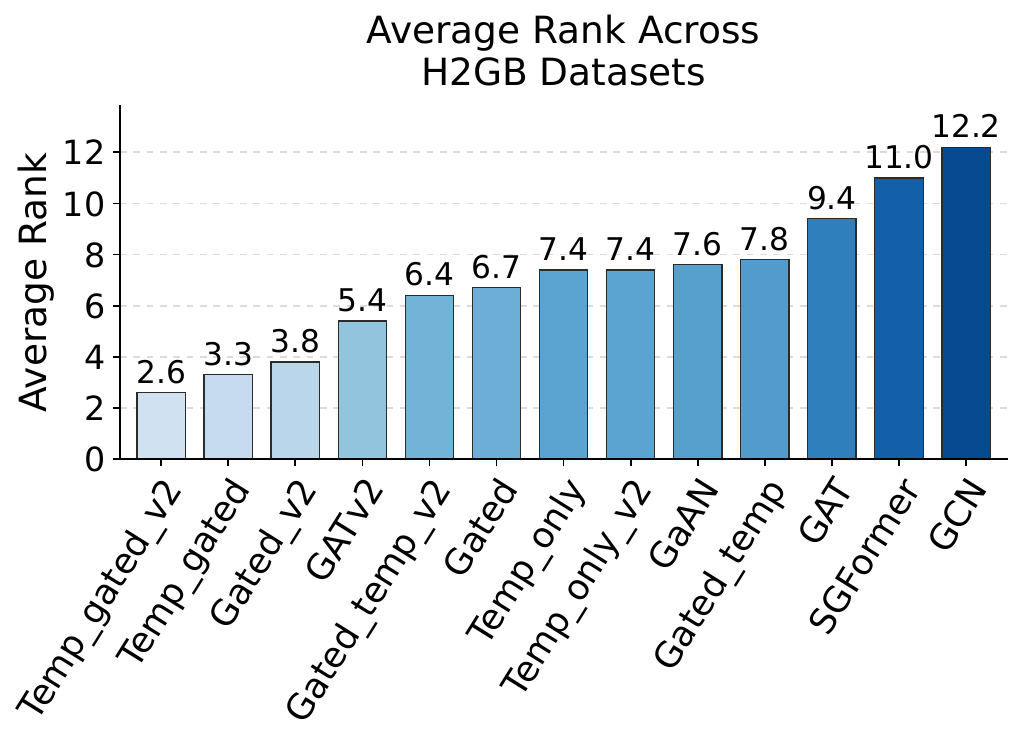}
\caption{Average ranks across the five H2GB benchmarks in Table~\ref{tab:h2gb_selected_combined_testacc}. Lower is better; ties receive average ranks.}
\Description{A bar chart summarizes the ranks of all thirteen methods across the five H2GB benchmarks.}
\label{fig:average-rank-h2gb}
\end{figure}

\begin{figure}[t]
\centering
\includegraphics[width=\textwidth]{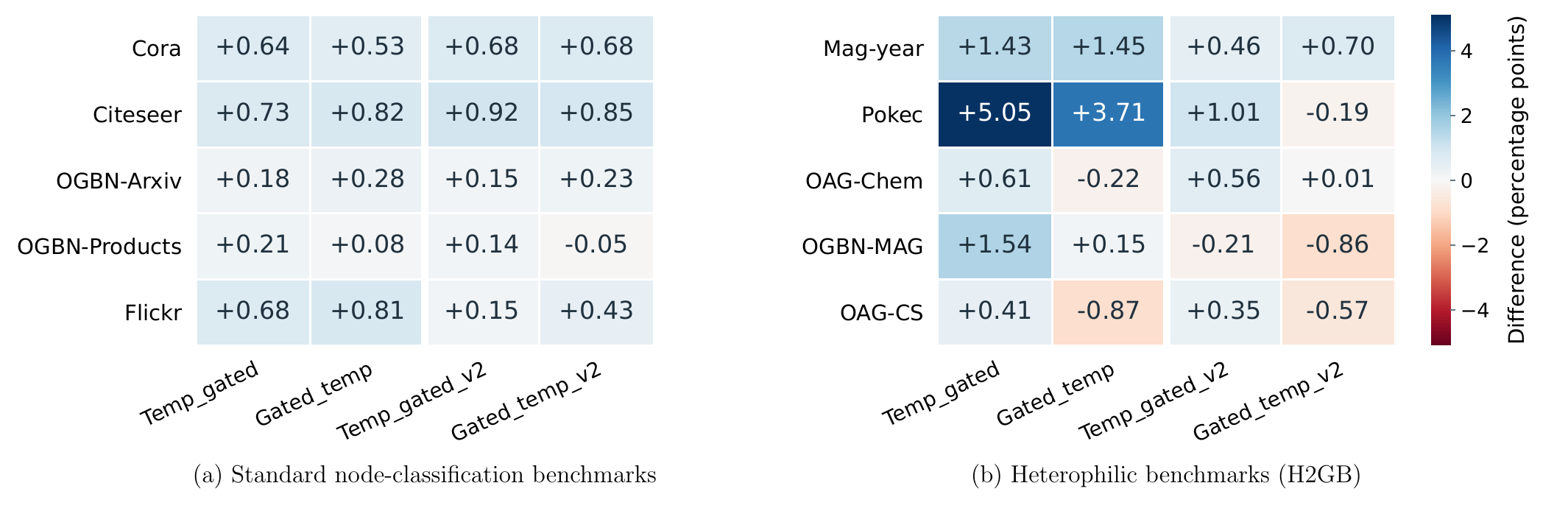}
\caption{\brownchange{Performance of the four combined variants relative to their GAT or GATv2 backbone. The left and right panels show the standard and heterophilic benchmarks, respectively. Values are differences in percentage points; blue indicates improvement and red decline. Both panels share a linear color scale centered at zero.}}
\Description{A heatmap reports mean score differences for four combined variants across five standard and five heterophilic benchmarks. Most cells show gains.}
\label{fig:combined-gain-heatmap}
\end{figure}

\brownchange{We also conduct controlled noise experiments (Appendix~\ref{app:noise}), whose results are \changed{qualitatively} consistent with the theoretical analysis in Section~\ref{sec:theory}.}

\section{Supporting Theoretical Analysis}\label{sec:theory}
We analyze two sources of unreliable attention, globally perturbed features and corrupted coordinates, under a two-class contextual stochastic block model (CSBM)~\cite{deshpande2018contextual}. The analysis isolates the mechanisms rather than modeling every experimental setting.

\subsection{Setting and Node-Level Metric}
Labels $y_i\in\{-1,+1\}$ are balanced, with edge probabilities $a/n$ within classes and $b/n$ between classes, where $a>b>0$. For a target node with $K>1$ neighbors, let $m=(a-b)/(a+b)>0$, so $\mathbb E[y_j\mid y_i,j\in\mathcal N(i)]=my_i$. The clean feature is $h_i^\star=y_i\mu$ for a fixed class signal $\mu\in\mathbb R^d$. We specialize the attention logit to
\begin{equation}\label{eq:l1score}
e_{ij}=-\sum_{\ell=1}^d w_\ell|h_{i\ell}-h_{j\ell}|,\qquad w_\ell>0.
\end{equation}
This weighted $\ell_1$ score is a special case of GATv2. It assigns larger logits to representation-similar neighbors. The condition $a>b$ maintains a positive structural label signal; as $a$ approaches $b$, that signal weakens.

\paragraph{Compatibility with GATv2.}\label{sec:score}
Let $D_w=\operatorname{diag}(w_1,\ldots,w_d)$ and let $\beta\in(0,1)$ be the negative slope of LeakyReLU. Choose
\begin{equation}
W=\begin{bmatrix}D_w&-D_w\\-D_w&D_w\end{bmatrix},
\qquad q=-\frac{1}{1-\beta}\mathbf 1_{2d}.
\end{equation}
The identity $\LeakyReLU_\beta(t)+\LeakyReLU_\beta(-t)=(1-\beta)|t|$ gives
\begin{equation}
q^\top\LeakyReLU_\beta(W[h_i\Vert h_j])
=-\sum_{\ell=1}^d w_\ell|h_{i\ell}-h_{j\ell}|.
\end{equation}
Thus, Eq.~\eqref{eq:l1score} is an instance of the GATv2 scoring form.

\paragraph{Node-level metric.}
We measure the separability of the attention-weighted neighbor-label score $R_i^{(T)}=\sum_{j\in\mathcal N(i)}\alpha_{ij}^{(T)}y_j$. Define $\mu_\pm(T)=\mathbb E[R_i^{(T)}\mid y_i=\pm1]$ and $v_\pm(T)=\operatorname{Var}(R_i^{(T)}\mid y_i=\pm1)$. Then
\begin{equation}\label{eq:snr}
\mathrm{SNR}(T)=\frac{\mu_+(T)-\mu_-(T)}{\sqrt{v_+(T)+v_-(T)}}.
\end{equation}
The gated score $\hat R_i^{(T)}$ and $\widehat{\mathrm{SNR}}(T)$ replace $\alpha$ by coefficients $\hat\alpha$ computed from gated features.

\subsection{Temperature under Global Gaussian Noise}
Suppose $h_i=y_i\mu+\sigma\varepsilon_i$, where $\varepsilon_i\sim\mathcal N(0,I_d)$ independently and $\sigma>0$ controls the noise level.
\begin{theorem}\label{thm:temp_gaussian}
Under this CSBM and the score in Eq.~\eqref{eq:l1score}, for target nodes with $K>1$,
\begin{equation}
\sup_{T>0}\mathrm{SNR}(T)\geq\mathrm{SNR}(1).
\end{equation}
As $\sigma\to\infty$, the logits become asymptotically uninformative about neighbor labels. In this limit, a temperature sufficiently large relative to the logit scale reduces attention-concentration variance while preserving the leading signal, yielding $\mathrm{SNR}(T)>\mathrm{SNR}(1)$.
\end{theorem}
See Appendix~\ref{app:proofs-temp} for the detailed proof.

This theorem explains the role of learnable temperature as controlling attention sharpness under global Gaussian noise. When the attention logits become dominated by label-independent noise, fixed-temperature attention may assign overly concentrated coefficients to noisy neighbors. A sufficiently large learned temperature smooths the attention allocation, reduces the variance of the attention-weighted label signal, and therefore improves the node-level SNR.

\subsection{Gating under Coordinate-Missing Noise}
Now let each coordinate be observed or replaced by noisy fill-in,
\begin{equation}\label{eq:missing}
\begin{aligned}
h_{i\ell}&=r_{i\ell}y_i\mu_\ell+(1-r_{i\ell})\xi_{i\ell},\\
r_{i\ell}&\sim\operatorname{Bernoulli}(1-\rho),\qquad
\xi_{i\ell}\sim\mathcal N(0,\tau^2),
\end{aligned}
\end{equation}
independently across nodes and coordinates. Here $\rho$ is the missing probability and $\tau$ is the fill-in noise level.
\begin{theorem}\label{thm:gate_missing}
Under Eq.~\eqref{eq:missing}, \changed{let $0<\rho<1$ and take $\tau\to\infty$ with $\tau/T\to0$. This gives the weak-logit regime}
\[\max_j|e_{ij}|=o_p(T),\qquad \max_j|\hat e_{ij}|=o_p(T).\] The oracle gate $g_{i\ell}=r_{i\ell}$ yields $\hat h_{i\ell}=r_{i\ell}y_i\mu_\ell$ and the gated logit $\hat e_{ij}=-\sum_\ell w_\ell|\hat h_{i\ell}-\hat h_{j\ell}|$. It preserves the class-separation logit gap while changing the leading logit variance from order $\tau^2$ to a bound independent of $\tau$. \changed{In this joint limit}, oracle gating therefore gives $\widehat{\mathrm{SNR}}(T)>\mathrm{SNR}(T)$ for sufficiently large $\tau$.
\end{theorem}
See Appendix~\ref{app:proofs-gate} for the detailed proof.

This theorem explains the role of gated graph attention under coordinate-missing noise. In the ideal oracle case, the gate removes unreliable noisy coordinates from the attention logit while preserving the leading class-separation logit gap. Under the weak-logit approximation, this reduces logit fluctuations and improves the first-order node-level SNR.

\subsection{Discussion}
These results explain complementary roles. Temperature controls attention concentration when global feature noise weakens the scoring signal, while gating filters unreliable coordinates. The gate result concerns an oracle under a first-order approximation, and the temperature result concerns a high-noise limit. Their empirical counterparts are the learned-variable trends in the controlled noise experiments. When both types of corruption occur, using the two mechanisms together can address both effects.

The distinction between the two noise settings also clarifies the roles of the learned variables. Global perturbations can make the entire feature-based ranking unreliable even when graph structure retains a class signal. Temperature adjusts how much the layer concentrates on that ranking. Coordinate corruption leaves some dimensions informative and makes others unreliable, motivating feature-wise filtering before scoring or after aggregation. In the theoretical gating construction, filtering occurs before scoring, corresponding to the Gate-First placement. The experiments additionally examine Post-Gate, which controls the aggregated response. Comparing these placements tests where filtering is most useful within the attention layer.

\section{Conclusion}
We introduced learnable temperature and gated graph attention for GAT and GATv2. Temperature adapts neighborhood attention sharpness, while gating modulates projected features or aggregated updates. Experiments on five standard node classification benchmarks and five heterophilic benchmarks show improvements over the corresponding backbones on most benchmarks, with combined variants often providing further gains. A CSBM analysis and controlled noise experiments support their complementary roles under global feature perturbations and unreliable coordinates. Extending the analysis to more general graph structures and applying the mechanisms to global graph attention are directions for future work.
\label{main-text-end}


\appendix
\section*{Appendix}
Appendix~\ref{app:settings} provides additional experimental settings for GCN, GaAN, and SGFormer. Appendix~\ref{app:noise} reports noise experiments to validate our theoretical results by showing how learnable temperature and gate responses vary with noise level. Appendix~\ref{app:proofs} presents the proofs of Theorems~\ref{thm:temp_gaussian} and~\ref{thm:gate_missing}.

\section{Experimental Settings}\label{app:settings}\label{app:baselines}
We retain the official splits and use validation performance to select the checkpoint for test evaluation. Standard benchmarks use Adam and H2GB uses the shared sampling protocol described in Section~\ref{sec:h2gb}; GaAN and SGFormer use AdamW on H2GB.

\paragraph{\brownchange{GCN}}
\brownchange{On \texttt{Cora} and \texttt{Citeseer}, GCN uses two layers with hidden dimensions of 64 and 128, respectively, together with dropout 0.5, learning rate 0.01, and weight decay $5\times 10^{-4}$. On \texttt{OGBN-Arxiv} and \texttt{OGBN-Products}, it uses three layers with hidden dimension 256, dropout 0.5, and learning rate 0.01, with weight decay set to $10^{-3}$ and 0, respectively. On \texttt{Flickr}, it uses two layers with hidden dimension 256, dropout 0.5, and learning rate 0.01. On \texttt{Mag-year} and \texttt{OGBN-MAG}, it uses four layers with hidden dimension 512, dropout 0.2, learning rate $10^{-3}$, and weight decay $10^{-5}$. On \texttt{Pokec}, it uses three layers with hidden dimension 256, dropout 0.5, learning rate $10^{-3}$, and weight decay $10^{-5}$. On \texttt{OAG-CS} and \texttt{OAG-Chem}, it uses three layers with hidden dimension 256, dropout 0.2, learning rate $5\times 10^{-3}$, and weight decay $10^{-5}$.}

\paragraph{\brownchange{GaAN and SGFormer}}
GaAN follows head-level gating with the corresponding GAT layer output width where applicable. It uses $d_a=8$ and $d_m=16$, with per-head value dimension $d_v=8$ on Cora, Citeseer, and Flickr, $d_v=16$ on OGBN-Arxiv and OGBN-Products, and hidden width divided by head count on H2GB. SGFormer uses one global attentive-propagation layer and a local graph branch. Its graph-branch mixing weights are 0.8 on Cora, 0.7 on Citeseer, and 0.5 elsewhere. Global attention operates within the sampled subgraph during mini-batch training. Model-specific architecture and optimizer settings are maintained in the implementation.

\section{Additional Experiments}\label{app:noise}

We further conduct controlled noise experiments to examine whether the learned variables behave consistently with the theoretical analysis. We use three standard node classification benchmarks, \texttt{Cora}, \texttt{Citeseer}, and \texttt{OGBN-Arxiv}, and evaluate both GAT-based and GATv2-based variants. \brownchange{Figure~\ref{fig:noise-analysis} reports} the learned temperature $T$ under isotropic Gaussian feature noise with strength $\sigma$. \brownchange{Figure~\ref{fig:gate-noise} reports} the gate behavior under coordinate-missing noise, where each coordinate is missing with probability $\rho$ and \changed{set to zero}. \changed{This experiment uses Gate-First gating with temperature fixed at $T=1$.}

The learned temperature generally increases as Gaussian noise becomes stronger\changed{. This trend can be observed on all three datasets, although its magnitude varies across datasets and layers}. This is consistent with Theorem~\ref{thm:temp_gaussian}. When global noise makes attention logits less reliable, a larger temperature smooths the attention distribution and reduces noise-driven concentration on individual neighbors.

For coordinate-missing noise, the gate mean tends to decrease \changed{in the first layer on \texttt{Cora} and \texttt{Citeseer} as $\rho$ increases}, indicating \changed{stronger average feature suppression}. \changed{The decrease is less pronounced or absent in the second layer. This difference may arise because the first layer directly processes corrupted input features, whereas the second layer operates on transformed representations.} \changed{On \texttt{OGBN-Arxiv}, the gate mean decreases most clearly in the intermediate layer and changes less in the first and third layers. These observations are qualitatively consistent with the feature-filtering mechanism analyzed in Theorem~\ref{thm:gate_missing}.}

\begin{figure}[t]
\centering
\begin{subfigure}{0.31\textwidth}
        \centering
        \includegraphics[width=\linewidth]{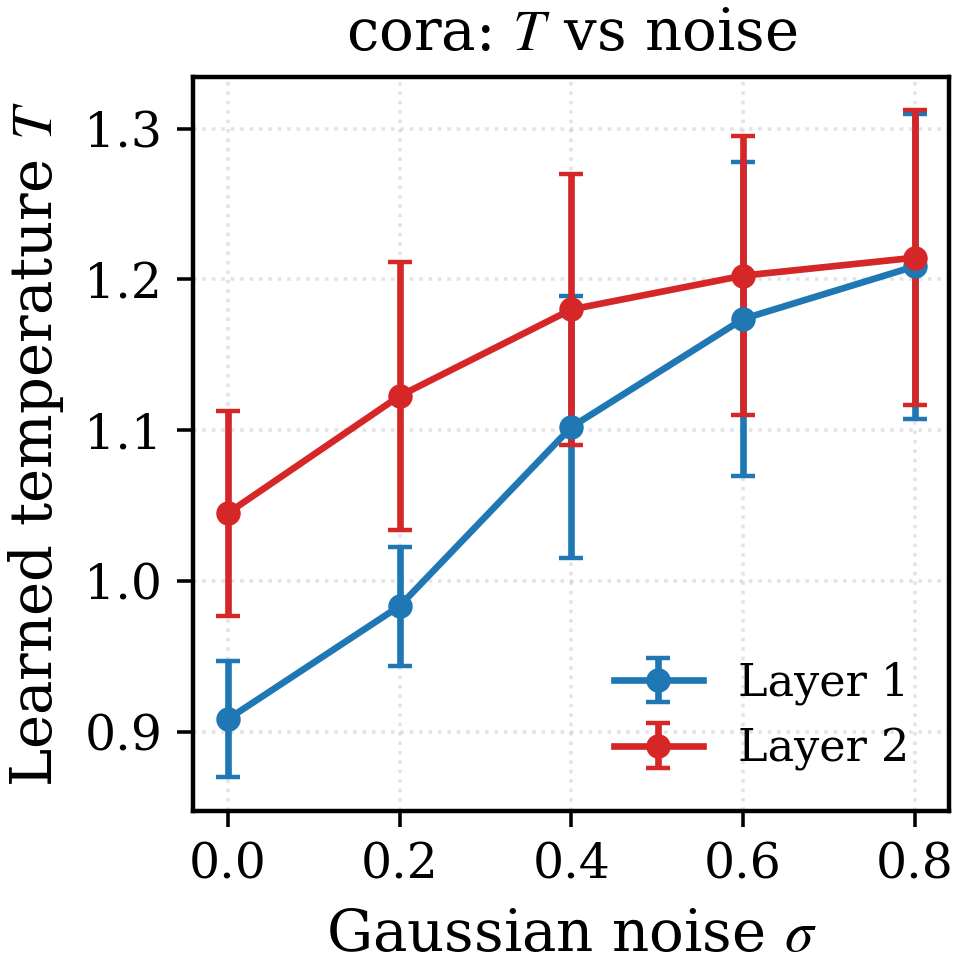}
        \caption{\texttt{Cora}, GAT}
    \end{subfigure}
\hfill
\begin{subfigure}{0.31\textwidth}
        \centering
        \includegraphics[width=\linewidth]{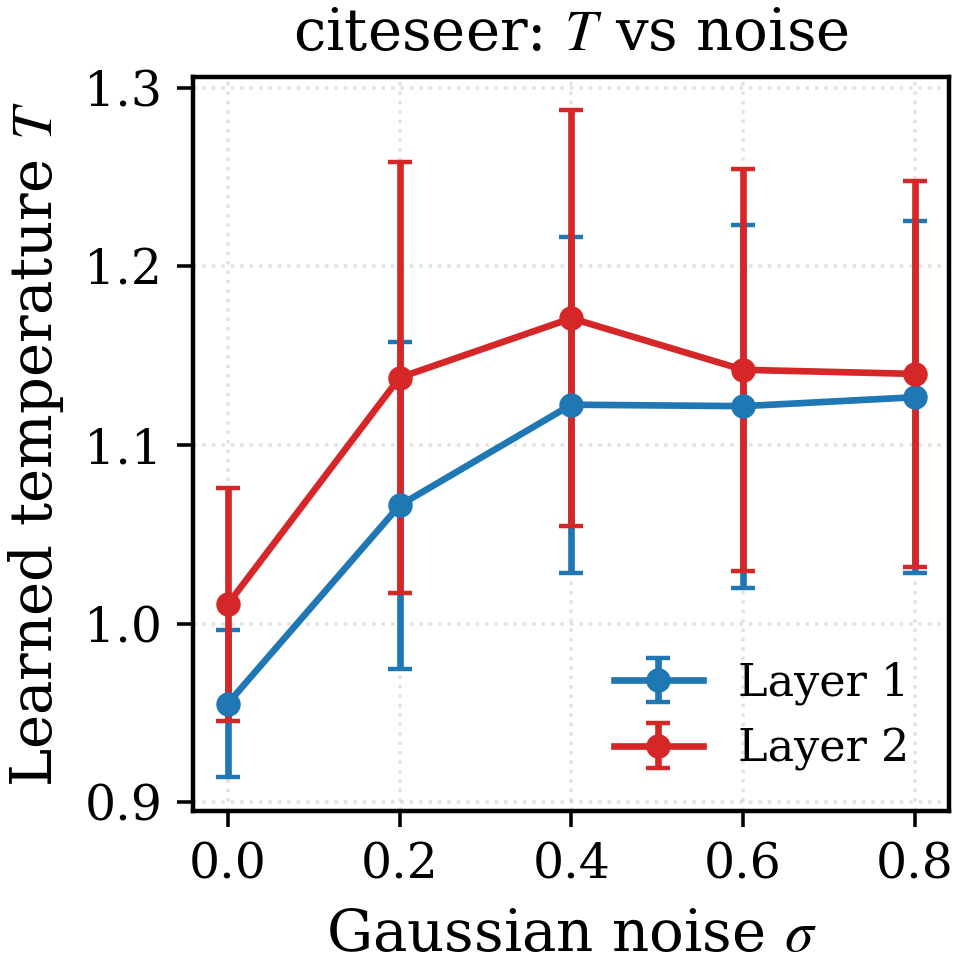}
        \caption{\texttt{Citeseer}, GAT}
    \end{subfigure}
\hfill
\begin{subfigure}{0.31\textwidth}
        \centering
        \includegraphics[width=\linewidth]{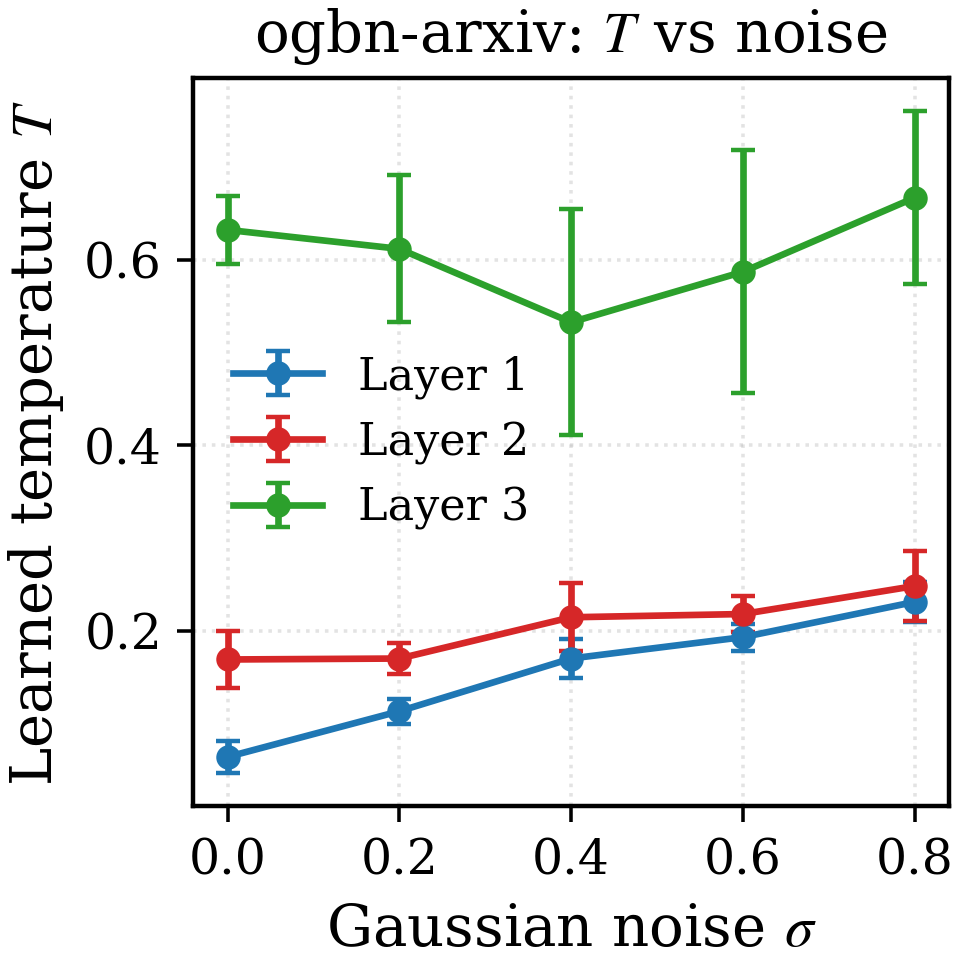}
        \caption{\texttt{OGBN-Arxiv}, GAT}
    \end{subfigure}
\par\medskip
\begin{subfigure}{0.31\textwidth}
        \centering
        \includegraphics[width=\linewidth]{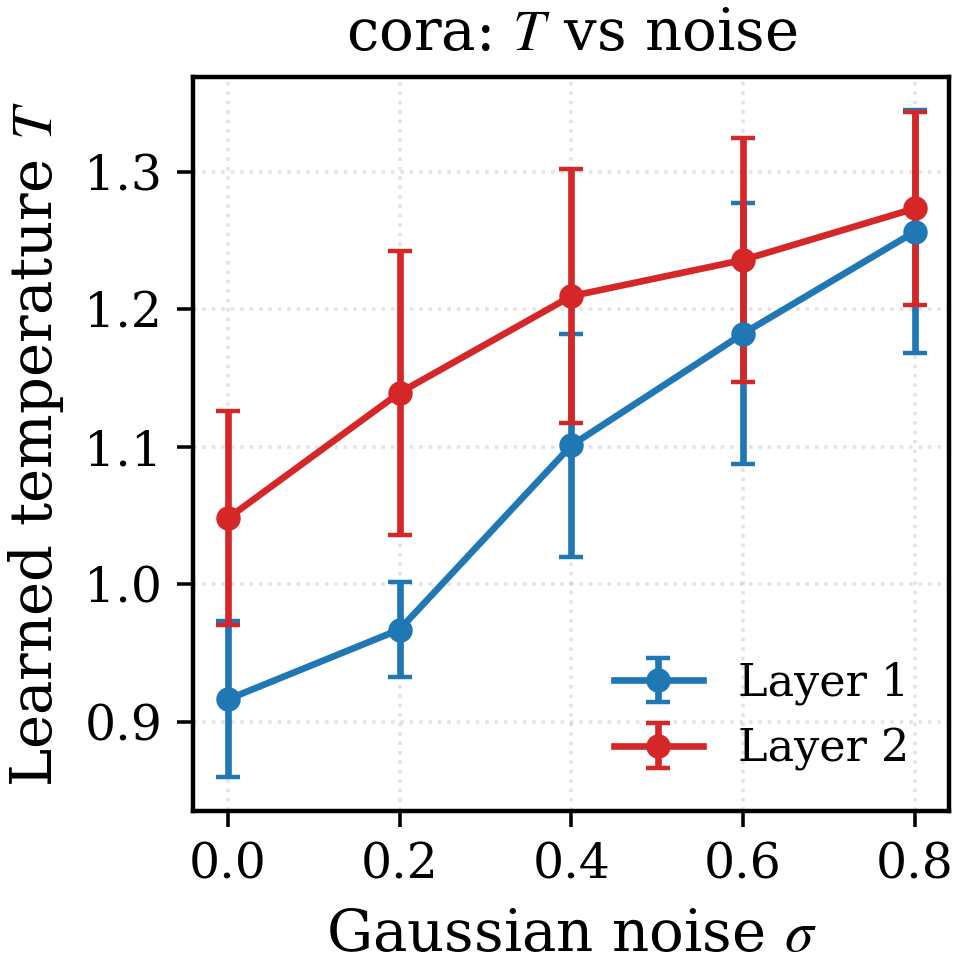}
        \caption{\texttt{Cora}, GATv2}
    \end{subfigure}
\hfill
\begin{subfigure}{0.31\textwidth}
        \centering
        \includegraphics[width=\linewidth]{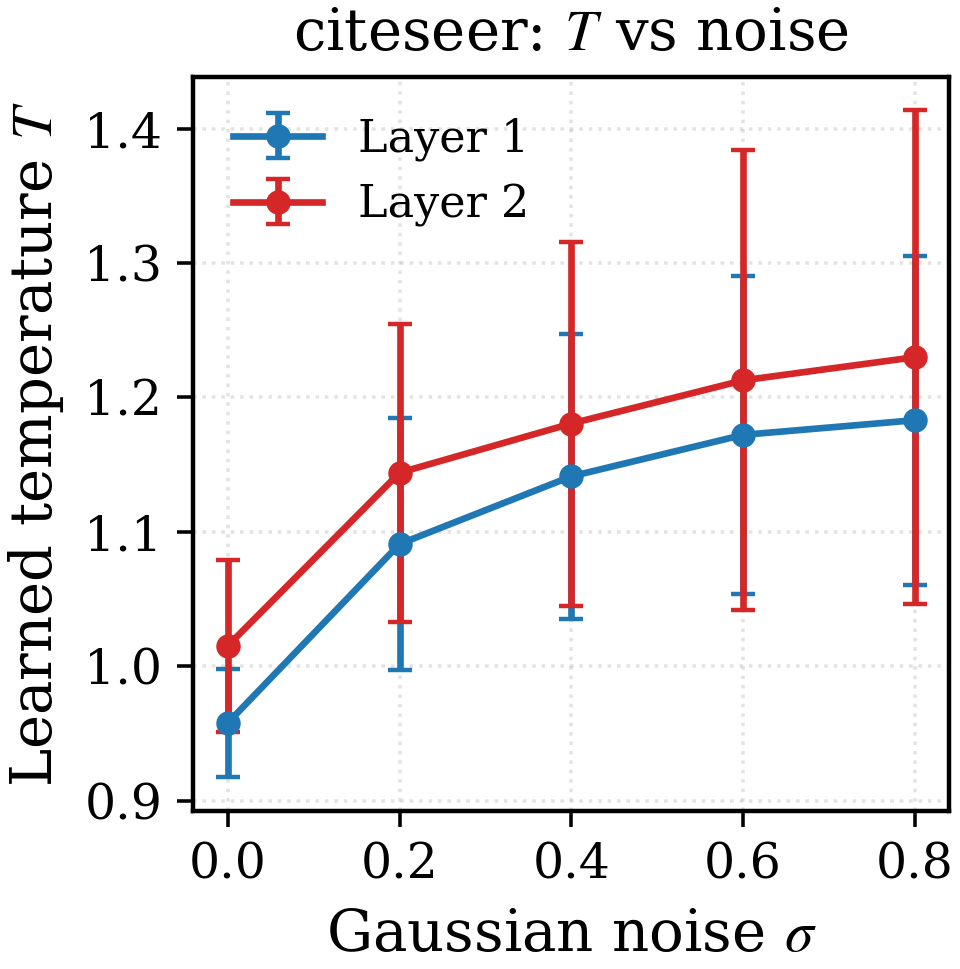}
        \caption{\texttt{Citeseer}, GATv2}
    \end{subfigure}
\hfill
\begin{subfigure}{0.31\textwidth}
        \centering
        \includegraphics[width=\linewidth]{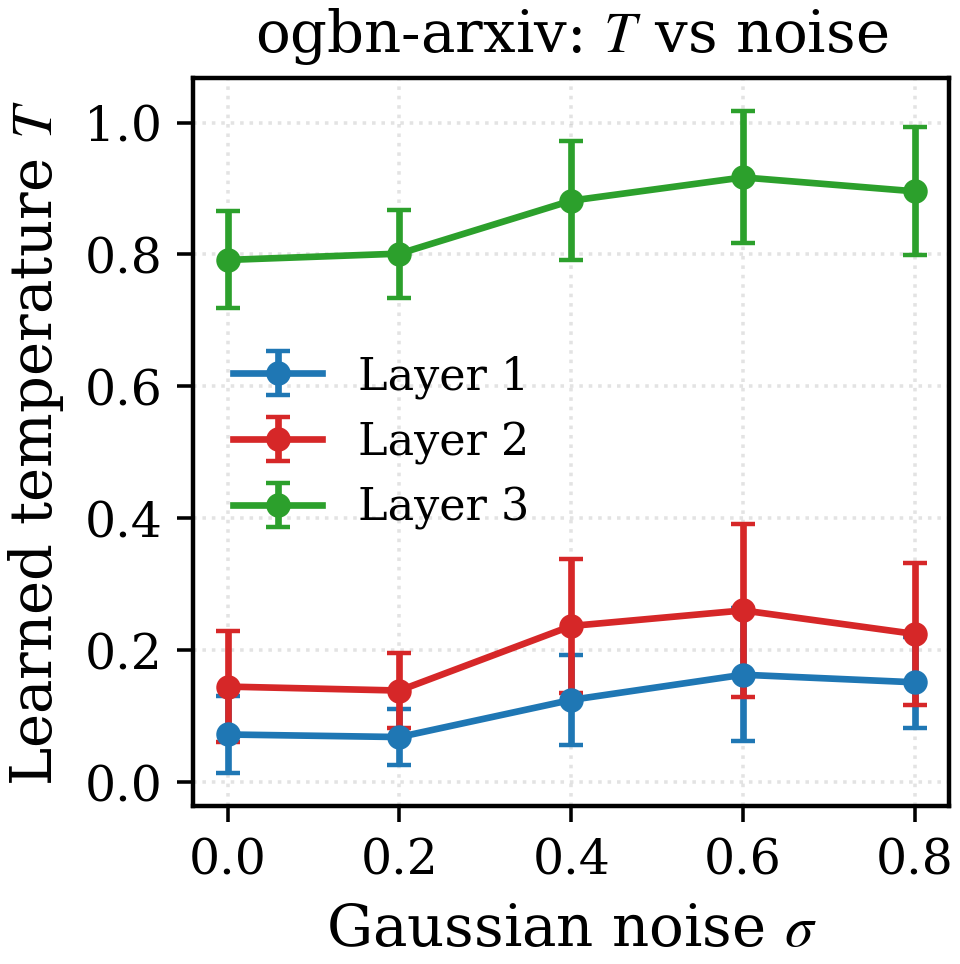}
        \caption{\texttt{OGBN-Arxiv}, GATv2}
    \end{subfigure}
\caption{\brownchange{Learned temperature under isotropic Gaussian feature noise. Columns show Cora, Citeseer, and OGBN-Arxiv; rows show GAT and GATv2.}}\label{fig:noise-analysis}
\Description{Six plots show layer-wise learned temperatures as Gaussian noise strength increases for GAT and GATv2 on three benchmarks.}
\end{figure}

\begin{figure}[t]
\centering
\begin{subfigure}{0.31\textwidth}
        \centering
        \includegraphics[width=\linewidth]{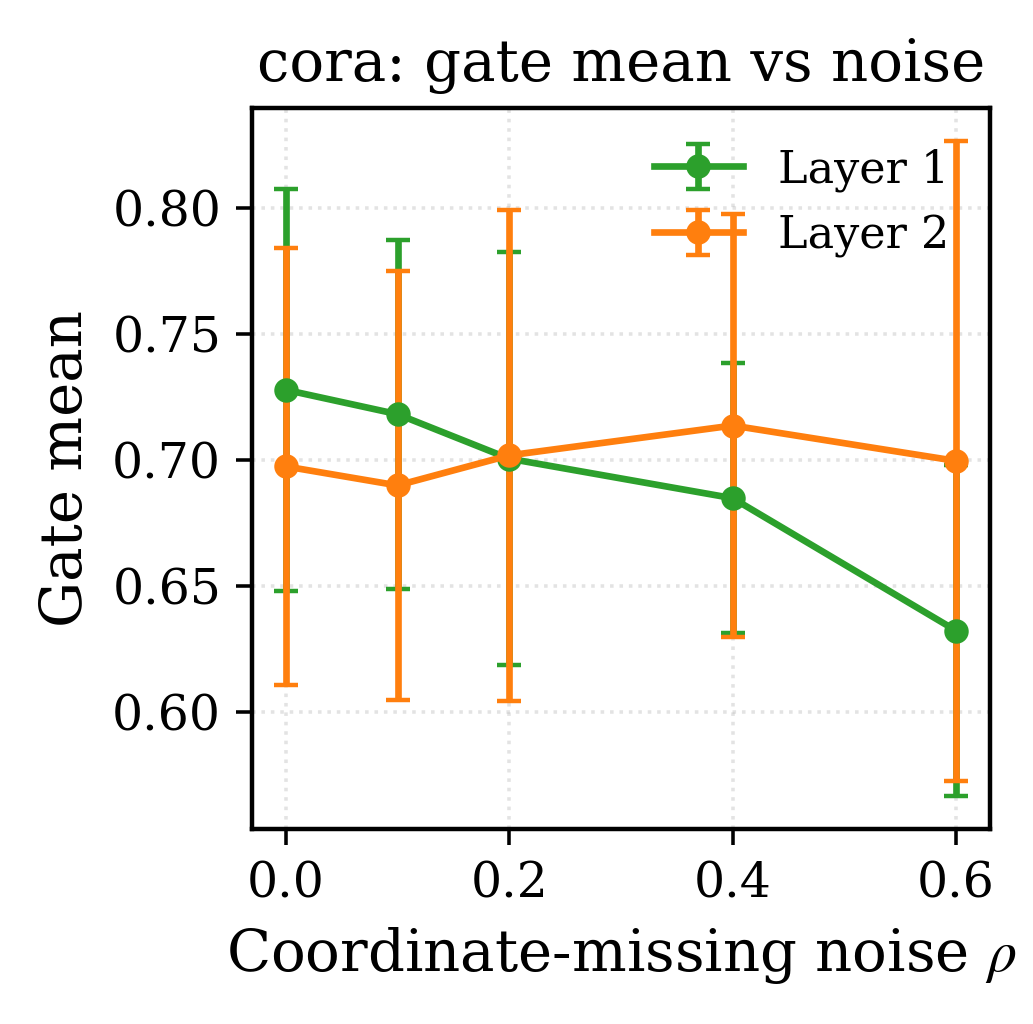}
        \caption{\changed{\texttt{Cora}, GAT}}
    \end{subfigure}
\hfill
\begin{subfigure}{0.31\textwidth}
        \centering
        \includegraphics[width=\linewidth]{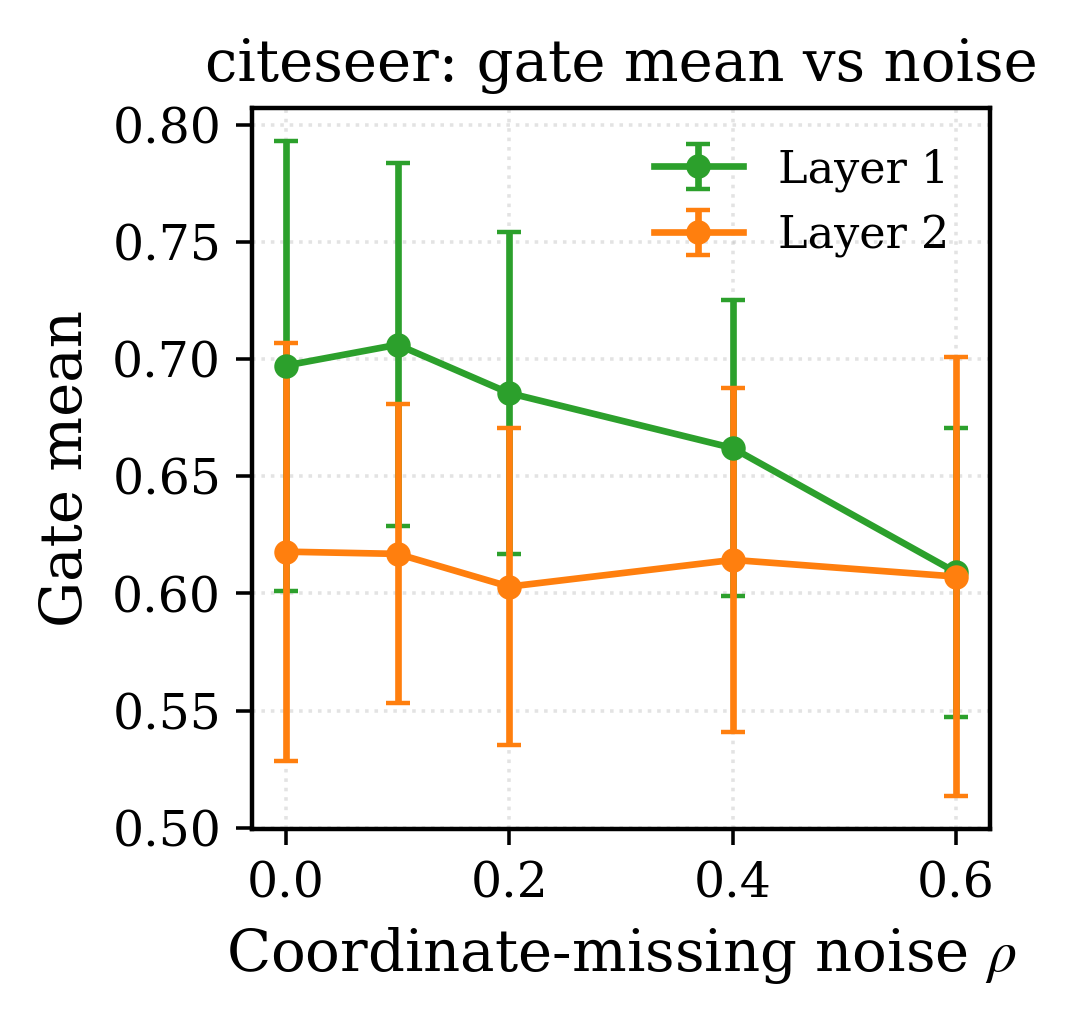}
        \caption{\changed{\texttt{Citeseer}, GAT}}
    \end{subfigure}
\hfill
\begin{subfigure}{0.31\textwidth}
        \centering
        \includegraphics[width=\linewidth]{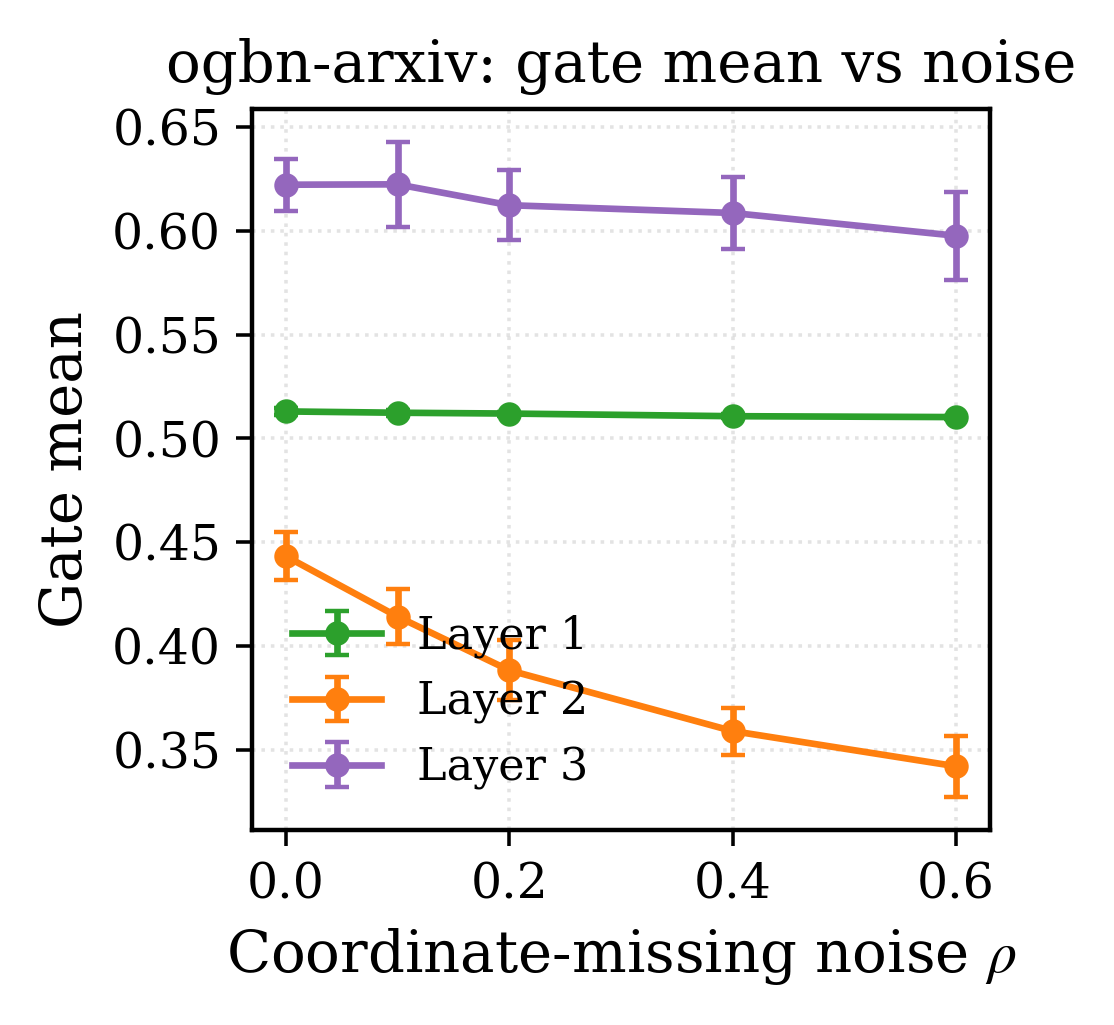}
        \caption{\changed{\texttt{OGBN-Arxiv}, GAT}}
    \end{subfigure}
\par\medskip
\begin{subfigure}{0.31\textwidth}
        \centering
        \includegraphics[width=\linewidth]{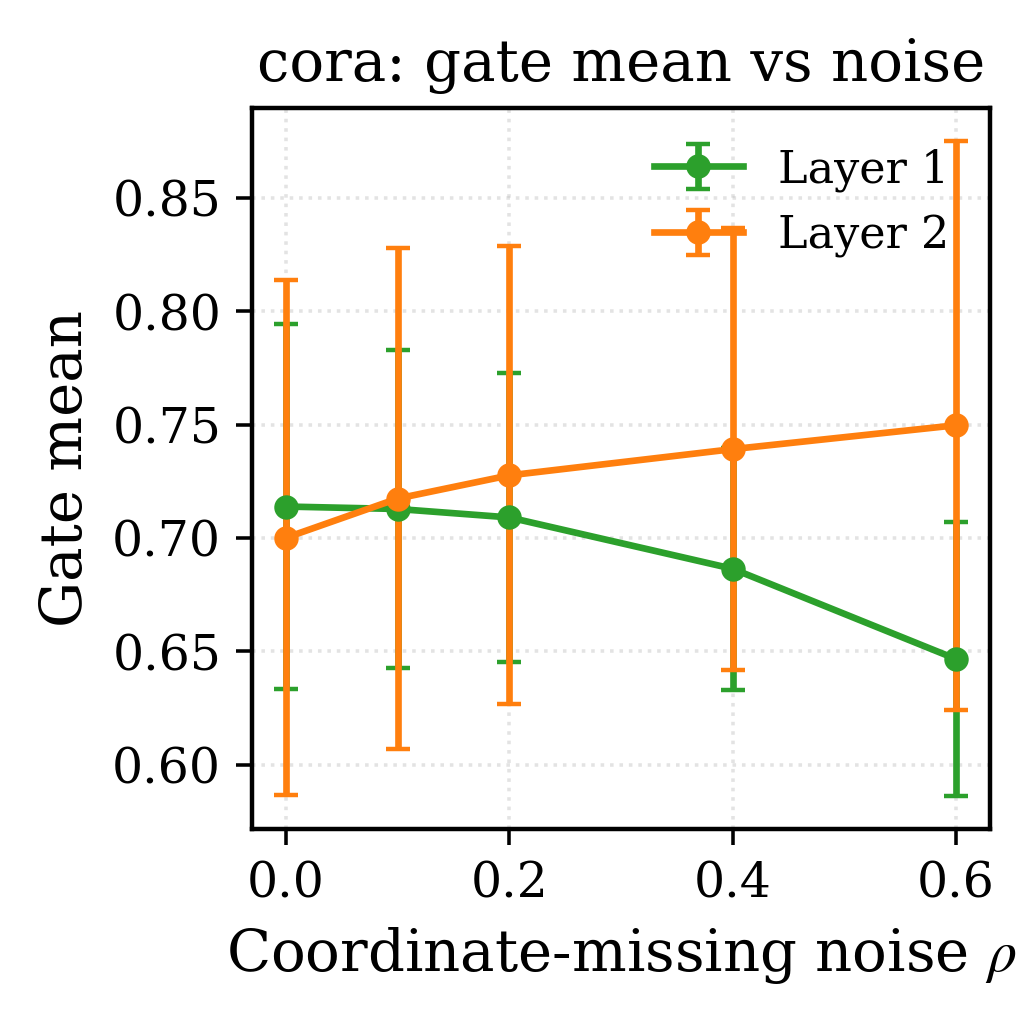}
        \caption{\changed{\texttt{Cora}, GATv2}}
    \end{subfigure}
\hfill
\begin{subfigure}{0.31\textwidth}
        \centering
        \includegraphics[width=\linewidth]{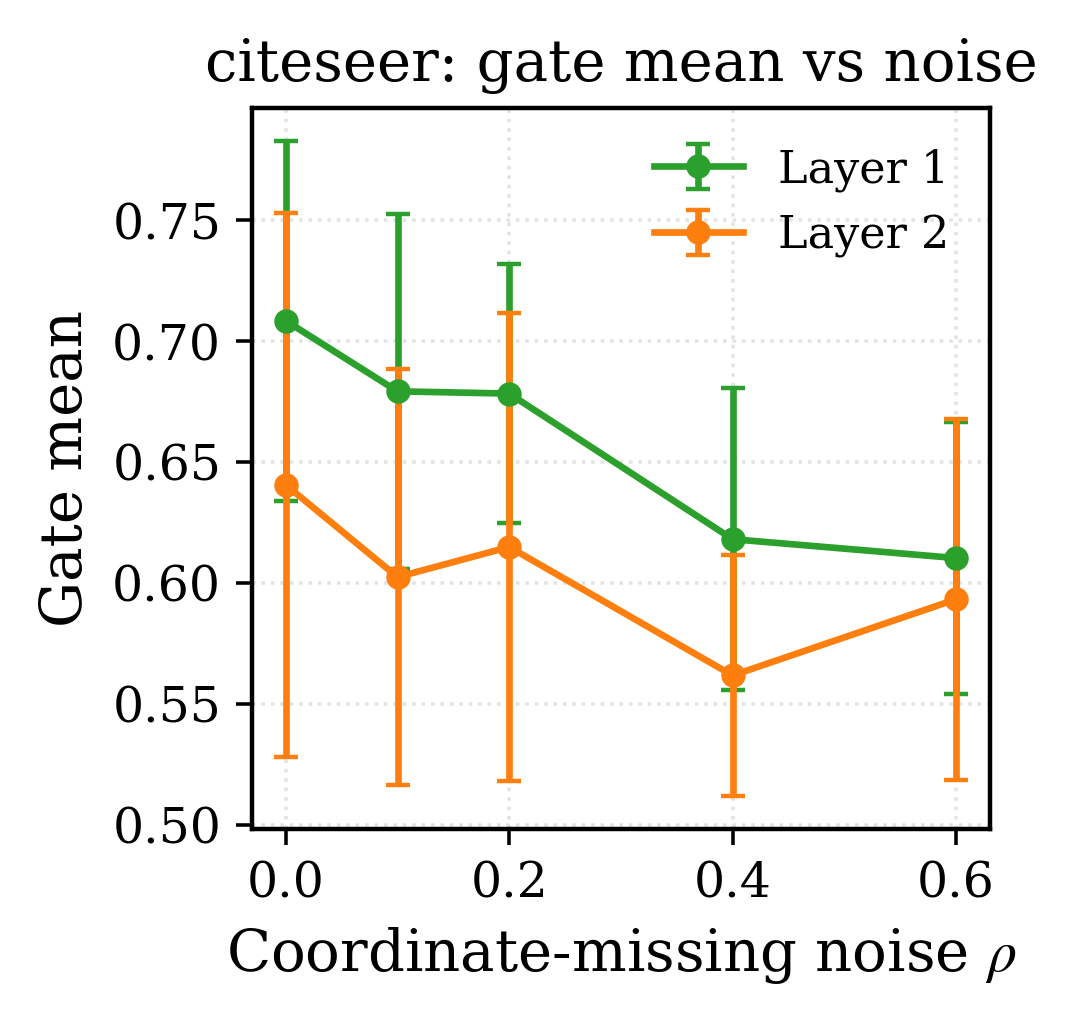}
        \caption{\changed{\texttt{Citeseer}, GATv2}}
    \end{subfigure}
\hfill
\begin{subfigure}{0.31\textwidth}
        \centering
        \includegraphics[width=\linewidth]{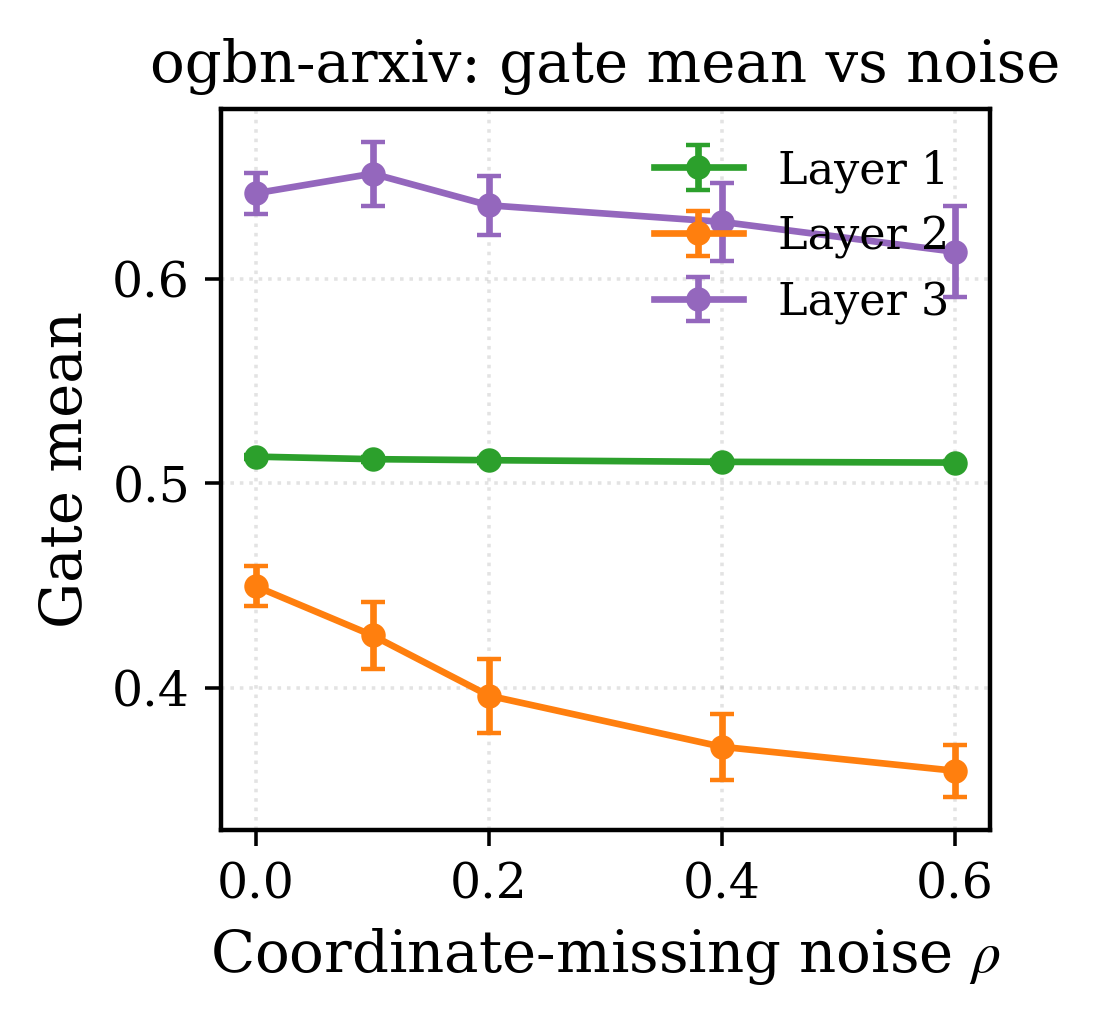}
        \caption{\changed{\texttt{OGBN-Arxiv}, GATv2}}
    \end{subfigure}
\caption{\brownchange{Mean gate activation \changed{with Gate-First gating and fixed $T=1$} under coordinate-missing noise. Columns show Cora, Citeseer, and OGBN-Arxiv; rows show GAT and GATv2. \changed{Markers and error bars show the mean and standard deviation over 20 runs.}}}\label{fig:gate-noise}
\Description{Six updated plots show mean Gate-First activations by layer under feature zero masking, with temperature fixed at one, for GAT and GATv2 on three benchmarks.}
\end{figure}

\section{Proofs of Theorems}
\label{app:proofs}

\subsection{Proof of Theorem~\ref{thm:temp_gaussian}}
\label{app:proofs-temp}
The first claim follows directly because the learnable-temperature family includes the fixed-temperature baseline $T=1$.

We next consider the limit $\sigma\to\infty$. For a neighbor $j\in\mathcal N(i)$,
\[
h_{i\ell}-h_{j\ell}
=
(y_i-y_j)\mu_\ell
+
\sigma(\varepsilon_{i\ell}-\varepsilon_{j\ell}).
\]
As $\sigma$ increases, the fixed class-dependent shift $(y_i-y_j)\mu_\ell$ is dominated by the Gaussian noise term. Therefore, the class-conditional distributions of the weighted $\ell_1$ logits become asymptotically identical, and the attention coefficients become asymptotically independent of the neighbor labels.

Conditioned on the attention coefficients, we have
\[
\mathbb E[R_i^{(T)}\mid \alpha^{(T)},y_i]
=
\sum_{j\in\mathcal N(i)}
\alpha_{ij}^{(T)}
\mathbb E[y_j\mid y_i]
=
m y_i,
\]
and
\[
\mathrm{Var}(R_i^{(T)}\mid \alpha^{(T)},y_i)
=
(1-m^2)
\sum_{j\in\mathcal N(i)}
\big(\alpha_{ij}^{(T)}\big)^2.
\]
Taking expectation over the attention coefficients gives
\[
\mathrm{Var}(R_i^{(T)}\mid y_i)
=
(1-m^2)
\mathbb E\!\left[
\sum_{j\in\mathcal N(i)}
\big(\alpha_{ij}^{(T)}\big)^2
\right]
+o(1).
\]

Thus, in the high-noise limit, the leading numerator of $\mathrm{SNR}(T)$ is independent of $T$, while the denominator is controlled by the attention concentration
\[
\mathbb E\!\left[
\sum_{j\in\mathcal N(i)}
\big(\alpha_{ij}^{(T)}\big)^2
\right].
\]
For any convex allocation $(u_1,\ldots,u_K)$,
\[
\sum_{j=1}^K u_j^2\ge \frac{1}{K},
\]
with equality if and only if $u_j=1/K$ for all $j$. Hence the variance is minimized by the uniform allocation.

Under the continuous Gaussian noise and $K>1$, the fixed-temperature logits at $T=1$ are unequal almost surely, so the corresponding attention allocation is non-uniform almost surely. Choosing a temperature sufficiently large relative to the logit scale makes the softmax allocation approach the uniform allocation, thereby reducing the variance term while preserving the leading signal term. This yields
\[
\mathrm{SNR}(T)>\mathrm{SNR}(1).
\]

\hfill $\square$\par

\subsection{Proof of Theorem~\ref{thm:gate_missing}}
\label{app:proofs-gate}

We prove the result for the oracle gate $g_{i\ell}=r_{i\ell}$. Fix a coordinate $\ell$ and define
\[
d_{ij,\ell}:=|h_{i\ell}-h_{j\ell}|,
\qquad
\hat d_{ij,\ell}:=|\hat h_{i\ell}-\hat h_{j\ell}|.
\]

Let
\[
A_\tau(t):=\mathbb E|t-\xi|,
\qquad
\xi\sim \mathcal N(0,\tau^2),
\]
and
\[
B_\tau:=\mathbb E|\xi_1-\xi_2|
=
\frac{2\tau}{\sqrt{\pi}},
\qquad
\xi_1,\xi_2\stackrel{\mathrm{i.i.d.}}{\sim}\mathcal N(0,\tau^2).
\]
For the ungated distance, if $y_j=y_i$, then
\[
\mathbb E[d_{ij,\ell}\mid y_j=y_i]
=
2\rho(1-\rho)A_\tau(|\mu_\ell|)
+
\rho^2B_\tau.
\]
If $y_j\neq y_i$, then
\[
\mathbb E[d_{ij,\ell}\mid y_j\neq y_i]
=
2(1-\rho)^2|\mu_\ell|
+
2\rho(1-\rho)A_\tau(|\mu_\ell|)
+
\rho^2B_\tau.
\]
Therefore,
\[
\mathbb E[d_{ij,\ell}\mid y_j\neq y_i]
-
\mathbb E[d_{ij,\ell}\mid y_j=y_i]
=
2(1-\rho)^2|\mu_\ell|.
\]

For the oracle-gated distance, if $y_j=y_i$, then
\[
\mathbb E[\hat d_{ij,\ell}\mid y_j=y_i]
=
2\rho(1-\rho)|\mu_\ell|,
\]
whereas if $y_j\neq y_i$, then
\[
\mathbb E[\hat d_{ij,\ell}\mid y_j\neq y_i]
=
2(1-\rho)|\mu_\ell|.
\]
Thus,
\[
\mathbb E[\hat d_{ij,\ell}\mid y_j\neq y_i]
-
\mathbb E[\hat d_{ij,\ell}\mid y_j=y_i]
=
2(1-\rho)^2|\mu_\ell|.
\]
Define the class-separation logit gaps
\[
\begin{aligned}
\Delta_e
&:=
\mathbb E[e_{ij}\mid y_j=y_i]
-
\mathbb E[e_{ij}\mid y_j\neq y_i],\\
\Delta_{\hat e}
&:=
\mathbb E[\hat e_{ij}\mid y_j=y_i]
-
\mathbb E[\hat e_{ij}\mid y_j\neq y_i].
\end{aligned}
\]
Since both logits are negative weighted sums of coordinate distances, summing over $\ell$ gives
\[
\Delta_{\hat e}
=
\Delta_e
=
2(1-\rho)^2
\sum_{\ell=1}^d w_\ell|\mu_\ell|.
\]

We next compare logit variances. For the oracle-gated feature,
\[
\hat d_{ij,\ell}\in\{0,|\mu_\ell|,2|\mu_\ell|\},
\]
so
\[
\mathrm{Var}(\hat d_{ij,\ell}\mid y_i,y_j)
\le
4\mu_\ell^2.
\]
By independence across coordinates,
\[
\mathrm{Var}(\hat e_{ij}\mid y_i,y_j)
\le
4\sum_{\ell=1}^d w_\ell^2\mu_\ell^2
=:C_{\mathrm g},
\]
where $C_{\mathrm g}$ is independent of $\tau$.

For the ungated feature, on the event $r_{i\ell}=r_{j\ell}=0$, which occurs with probability $\rho^2$,
\[
d_{ij,\ell}=|\xi_{i\ell}-\xi_{j\ell}|.
\]
Since
\[
\xi_{i\ell}-\xi_{j\ell}
\sim
\mathcal N(0,2\tau^2),
\]
we have
\[
\mathrm{Var}(|\xi_{i\ell}-\xi_{j\ell}|)
=
2\left(1-\frac{2}{\pi}\right)\tau^2.
\]
Hence
\[
\mathrm{Var}(e_{ij}\mid y_i,y_j)
\ge
2\rho^2
\left(1-\frac{2}{\pi}\right)
\tau^2
\sum_{\ell=1}^d w_\ell^2.
\]
Thus, the oracle gate preserves the class-separation logit gap while changing the logit variance from order $\tau^2$ to order $1$.

\begingroup\color{editbrown}
To pass from logits to the node score, write $Y=(y_j)_{j\in\mathcal N(i)}$, $U=K^{-1}\sum_j y_j$, and $V=K^{-1}\sum_j e_{ij}(y_j-U)$. The first-order score is $U+V/T$. The equal logit gaps give
\[
\mathbb E[V\mid y_i,Y]=\mathbb E[\hat V\mid y_i,Y]
=\frac{y_i\Delta_e}{2}(1-U^2).
\]
Thus the means and the covariance terms with $U$ agree. Write $\xi_{j\ell}=\tau Z_{j\ell}$ and $e_{ij}=\tau B_j+O(1)$, where
\[
B_j=-\sum_\ell w_\ell|(1-r_{i\ell})Z_{i\ell}-(1-r_{j\ell})Z_{j\ell}|.
\]
The vector $B$ is independent of the labels. With $\bar B=K^{-1}\sum_jB_j$, define
\[
c=\frac{1-m^2}{K^2}\mathbb E\sum_j(B_j-\bar B)^2>0.
\]
Strict positivity follows from $K>1$ and nonconstant neighbor fill-in noise, conditional on an observed target. Consequently, $\operatorname{Var}(V\mid y_i)=c\tau^2+O(\tau)$, whereas $\operatorname{Var}(\hat V\mid y_i)=O(1)$.

For the exact softmax score, the leading second-order coefficient is $W_0=(2K)^{-1}\sum_j(B_j-\bar B)^2(y_j-U)$. Conditional on $B$, both its mean and its covariance with $U$ vanish. The remaining corrections to the mean difference and the variance difference are $O(\tau/T^2+\tau^3/T^3)=o(\tau^2/T^2)$. Both class-conditional variances tend to $v_0=(1-m^2)/K$, while the means tend to $\pm m$. Hence
\[
\widehat{\mathrm{SNR}}(T)-\mathrm{SNR}(T)
=\frac{mc}{\sqrt{2}\,v_0^{3/2}}\frac{\tau^2}{T^2}
+o(\tau^2/T^2)>0.
\]
\endgroup

\hfill $\square$\par

\clearpage
\end{document}